\documentclass[final,12pt]{colt2024} 

\title[Model Selection for Average Reward w.a.t. Utility Maximization in Games]{Model Selection for Average Reward RL with Application to Utility Maximization in Repeated Games}
\usepackage{times}

\usepackage[round]{natbib}

\usepackage{makecell}
\usepackage{lipsum}
\usepackage{etoolbox} 
\usepackage{algorithm}
\usepackage{algpseudocode}
\usepackage[titletoc]{appendix}

\providecommand{\graf}[1]{}

\newcommand{\parhead}[1]{\noindent{\textbf{#1}}}

\usepackage{amsmath,amsfonts,amssymb}

\newtheorem{theorem}{Theorem}[section]
\newtheorem{lemma}[theorem]{Lemma}
\newtheorem{corollary}[theorem]{Corollary}
\newtheorem{proposition}[theorem]{Proposition}
\newtheorem{definition}[theorem]{Definition}

\newtheorem{remark}[theorem]{Remark}
\newtheorem{assumption}[theorem]{Assumption}

\newcommand{\pmdt}{\textsf{PMEVI-DT}}
\newcommand{\Reg}{\mathsf{Reg}}
\newcommand{\cfive}{c_{h^*}^5}
\newcommand{\sR}{\mathsf{R}}
\newcommand{\fb}{\mathfrak{b}}
\newcommand{\K}{\mathsf{K}}
\newcommand{\bear}{\textsf{MRBEAR}}
\newcommand{\spn}{\mathsf{sp}}
\newcommand{\cH}{\mathcal{H}}
\newcommand{\cA}{\mathcal{A}}
\newcommand{\dB}{\mathsf{dB}^{+1}}
\newcommand{\cS}{\mathcal{S}}
\newcommand{\cC}{\mathcal{C}}
\newcommand{\cB}{\mathcal{B}}

\newcommand{\cR}{\mathcal{R}}

\newcommand{\cG}{\mathcal{G}}

\newcommand{\cO}{\mathcal{O}}
\newcommand{\cM}{\mathcal{M}}

\newcommand{\cI}{\mathcal{I}}
\newcommand{\cN}{\mathcal{N}}

\newcommand{\pB}{\mathbb{B}}
\newcommand{\balg}{\mathsf{Alg}}
\newcommand{\bb}{\mathsf{B}}
\newcommand{\bS}{\boldsymbol{S}}

\newcommand{\sumt}{\sum_{t =1 }^{T}}

\newcommand{\sumk}{\sum_{k =1 }^{\mathsf{K}(T)}}

\newcommand{\1}{\mathbf{1}}

\newcommand{\R}{\mathbb{R}}
\newcommand{\E}{\mathbb{E}}
\newcommand{\Prob}{\mathbb{P}}
\newcommand{\N}{\mathbb{N}}
\newcommand{\bE}{\mathbf{E}}
\newcommand{\bP}{\mathbf{P}}
\newcommand{\bU}{\boldsymbol{U}}
\newcommand{\bA}{\boldsymbol{A}}
\newcommand{\bB}{\boldsymbol{B}}

\newcommand{\tg}{\tilde g}

\newcommand{\Ot}{\tilde O}

\coltauthor{%
 \Name{Alireza Masoumian} \Email{amasoumi@ualberta.ca}\\
 \addr Department of Computing Science, University of Alberta / Amii
 \AND
 \Name{James R. {Wright}} \Email{james.wright@ualberta.ca}\\
 \addr Department of Computing Science, University of Alberta / Amii%
}

\begin{document}

\maketitle

\begin{abstract}
In standard RL, a learner attempts to learn an optimal policy for a Markov Decision Process whose structure (e.g. state space) is known. In online model selection, a learner attempts to learn an optimal policy for an MDP knowing only that it belongs to one of $M >1$ model classes of varying complexity. Recent results have shown that this can be feasibly accomplished in episodic online RL. In this work, we propose $\bear$, an online model selection algorithm for the average reward RL setting which is based on the idea of regret balancing and elimination. The regret of the algorithm is in $\tilde O(M C_{m^*}^2 \bb_{m^*}(T,\delta))$ where $C_{m^*}$ represents the complexity of the simplest well-specified model class and $\bb_{m^*}(T,\delta)$ is its corresponding regret bound. This result shows that in average reward RL, like the episodic online RL, the additional cost of model selection scales only linearly in $M$, the number of model classes. We apply $\bear$ to the interaction between a learner and an opponent in a two-player simultaneous general-sum repeated game, where the opponent follows a fixed unknown limited memory strategy. The learner's goal is to maximize its utility without knowing the opponent's utility function. The interaction is over $T$ rounds with no episode or discounting which leads us to measure the learner's performance by average reward regret. In this application, our algorithm enjoys an opponent-complexity-dependent regret in $\tilde O(M(\mathsf{sp}(h^*) B^{m^*} A^{m^*+1})^{\frac{3}{2}} \sqrt{T})$, where $m^*\le M$ is the unknown memory limit of the opponent, $\mathsf{sp}(h^*)$ is the unknown span of optimal bias induced by the opponent, and $A$ and $B$ are the number of actions for the learner and opponent respectively. We also show that the exponential dependency on $m^*$ is inevitable by proving a lower bound on the learner's regret.
\end{abstract}

\section{Introduction}
\graf{Most work assumes known model class}
Most work in average reward-based reinforcement learning assumes that the underlying model class---i.e., state space and actions---is known, even if the probabilities that define rewards and the transitions between states must be learned.  However, in practice the definitions of the state space must be specified.  This leads to a tradeoff between the richness of the model and the hardness of learning within it.  A model class that encodes every possible observation in its states will have a very large set of states, making it very challenging to learn an optimal policy; conversely, a model class that encodes too little information into its states may not be rich enough to accurately describe the behavior of the domain.

\graf{Online model selection in episodic setting is asymptotically better than learning in most-general class}
One way to resolve this tradeoff is by selecting the model class during learning based on interactions with the environment.  In the episodic regret setting, it has been shown that online model selection using $M$ different model classes (at least one of which is assumed to be well-specified) can attain a regret bound that is only a factor of $M$ worse than learning using the simplest well-specified model class from the collection.  Since the most general model class can have regret bounds that are superlinearly (often exponentially) greater than those of the optimal model class, this results in an asymptotic improvement over the naive approach of learning in a single complex model class that is guaranteed to be well-specified.
\graf{More details about the episodic case}
This asymptotic improvement is attained by using a meta-algorithm which selects one model class per time step, and runs one step of a learning algorithm on the selected model class.  Over time, an optimal policy is learned for a well-specified model class.  The meta-algorithm typically uses a criterion either from a bandit algorithm or a technique known as \emph{regret balancing} \citep{abbasi2020balancebanditrl, pacchiano2020regretbalanceelim} to achieve its guarantee.

\graf{Our question: Can we achieve a similar regret guarantee in the average reward setting?}
\noindent{ \textbf{Main Question.}} The main question that we ask is whether we can achieve a similar regret bound for online model selection in the average reward setting; namely, one which is only a linear factor of $M$ worse than the regret bound of the optimal model class.  
We answer this question affirmatively, by designing an algorithm based on regret balancing called $\bear$.
The average reward setting presents challenges that are not present in the episodic setting.
Any planning effects from a policy must be fully realized by the end of an epoch; so running a base learning algorithm in a selected model class for a single epoch is sufficient to generate an accurate sample of a policy's performance.  In contrast, a single iteration (i.e., a single action in the MDP) is unlikely to give an accurate sample of a policy's performance; but since the average reward setting has no notion of episodes, the algorithm must decide on a number of iterations to commit to a model class after each selection.  This results in a stricter notion of regret \eqref{Reg}.

\graf{Example: Our application setting}
\noindent{ \textbf{Application.}} We apply our model selection result to the repeated game setting against an opponent with limited memory, with the goal of maximizing the learner's cumulative utility. This goal, together with having no discounting or reset during the $T$ interactions, leads us to use average reward regret to evaluate the learner's performance.
The opponent can condition their choice in each game on the past $0 \le m^* < M$ plays, where the upper limit $M$ is known to us but the true limit $m^*$ is not.  A naive approach would be to learn a policy in the induced Markov Decision Process (MDP) class where each state contains the last $M$ choices; however, this would result in an MDP with exponentially more states than the true induced MDP, which has one state for each possible sequence of $m^*$ past choices.  Model selection instead allows us to learn the opponent's memory limit simultaneously with learning an optimal policy, by treating the state space induced by each potential memory limit as a separate model class.
\graf{Why average regret}
Average regret with respect to the optimal policy against the opponent is a more natural performance measure in this setting than episodic regret or PAC guarantees, since all of the rewards during the course of learning are important.
\graf{Why not backward induction?}
When both players are assumed to have unbounded memory and common knowledge of rationality (i.e., both are utility-maximizers, and know that the other is a utility-maximizer, and know that the other knows, etc.), an optimal strategy for each agent can be found using backward induction\citep{shoham2008multiagent}.
However, in this work we relax both assumptions.
Common knowledge of rationality in practice requires both agents to know the others' utility; in practice, this will rarely be true.  We further assume that the opponent's policy depends only on a fixed number of past actions taken in the game, but that this memory limit is not known to the learner.  Thus, our problem can be understood as finding the best response, in the repeated game, to an opponent's fixed, finite-memory strategy.
Note that we are interested in finding the optimal policy for the full repeated game, which in general might differ from a policy that best-responds to the opponent's action in each stage game individually.
If we knew the opponent's memory bound, then any procedure that learned a policy with a sublinear regret bound would eventually converge to a best response to the opponent's policy in the full repeated game.
Since this memory bound is unknown, we learn it simultaneously with learning a policy by solving a model selection problem using \bear.

\graf{Contribution and Organization}
The rest of the paper is structured as follows.  After surveying related work and defining our setting, we describe the base learning algorithm that we use in section~\ref{Bases Algorithms}.  In section~\ref{algorithms} we describe our proposed model selection algorithm.  Section~\ref{RegAnalysis} proves a regret bound for our algorithm that depends on the unknown complexity of the MDP. Section~\ref{application to games} specializes our results to the repeated-game setting, giving bounds that depend on the unknown memory and complexity of the opponent's strategy, including a minimax lower bound on the performance of any algorithm.
We wrap up with a discussion of our conclusions in section~\ref{conclusions}.



\subsection{Related Work} \label{related works}
\parhead{Average Reward.}
\citet{auer2008near} achieve a regret bound of $O(DS\sqrt{AT})$ with their seminal algorithm \textsf{UCRL2}.
Further work improves the dependency on the diameter $D$ by considering different techniques \citep{fruit2020Bern,bourel2020localDiam,wei2020model, ortner2020regret, talebi2018variance}. Diameter based algorithms usually assume communicating MDPs. However, \citet{fruit2018TUCRL} propose a UCRL-based algorithm for weakly communicating MDPs. When rewards are bounded, it is not hard to show that $\spn(h^*) \leq D$.
There even exist weakly communicating MDPs with infinite diameter, but with finite $\spn(h^*)$.
These show that results that capture $\spn(h^*)$ tend to be stronger \citep{fruit2018SCAL,zhang2019regret,boone2024achieving}.
The leading algorithm is the very recently introduced PMEVI, which efficiently achieves the optimal regret bound of $\tilde O(\sqrt{\spn(h^*)} SAT)$ without prior knowledge of $\spn(h^*)$ \citep{boone2024achieving}.

\parhead{Model Selection.}
There is a rich literature on model selection for statistical learning \citep{vapnik2006estimation,lugosi1999adaptive,massart2007concandmodel,bartlett2002model} and online learning \citep{freund1997decision,foster2017parameter-free}. \citet{abbasi2020balancebanditrl,pacchiano2020regretbalanceelim} use online misspecification tests, like checking whether the empirical regret satisfies known bounds, to determine which of the model classes are well specified. Some work assumes nested structure on the model classes \citep{ghosh2021modelgenericRL,foster2019model}, but some others do not \citep{agarwal2017corralling}. The cost of model selection under different assumptions can be either multiplicative \citep{pacchiano2020regretbalanceelim} or additive \citep{ghosh2021modelgenericRL} to the regret of the best well-specified model class \citep{marinov2021pareto,krishnamurthy2024costlessmodelselectio}. All of the above papers analyze episodic regret; to the best of our knowledge, there is no work in model selection for the average reward setting.
Our results are a first step in extending the model selection approach to the average reward setting.

\parhead{Learning in Markov Games.}
Another related direction to this work is learning in stochastic games \citep{wei2017online,zhong2021can,xie2020learningSG}. The setting in this paper differs from the previous works either in the notion of regret, assuming different levels of memory for the opponent, or the measure of complexity. An attractive recent line of research studies regret minimization when the opponent runs online learning algorithms \citep{brown2024learning, deng2019strategizing,braverman2018selling}. In our work, we bound average reward regret, a more demanding benchmark which better captures utility maximization in repeated play than external regret. A recent paper by \citet{assos2024maximizing} shows utility maximization against an opponent deploying multiplicative weights updates (MWU) algorithms is tractable in zero-sum repeated games, while there is no Fully Polynomial Time Approximation Scheme (FPTAS) for utility maximization against a best-responding opponent in general sum games. We circumvent this impossibility result by assuming that the opponent has limited memory. Another very recent work is \citep{nguyen2024learning} in which a bounded memory for the opponent is assumed and the performance of the learner is measured by policy regret \citep{arora2018policy}. However \citet{nguyen2024learning} model the interaction as a Stackelberg game, in episodic Markov Games. We study average reward in simultaneous repeated games which are arguably more complicated settings. \citet{zhong2021can} use reinforcement learning to solve the Stackelberg equilibrium for myopic followers. \citet{wei2017online} study stochastic games under average reward using the diameter as the complexity measure. Our own work proves bounds based on $\spn(h^*)$ which are tighter than those based on diameter.

A more detailed literature review is in appendix \ref{extensive related works}.

\section{Notation and Problem Setting}
\label{MDP Setting}
A Markov Decision Process (MDP) $\cM$ is a tuple, $(\cS, \cA, P, r, T, \mu)$ where $\cS$ and $\cA$ are the state and action spaces with the cardinality $|\cS| = \bS$ and $|\cA| = \bA$.  The reward $R_t = r(s_t,a_t)$ is obtained by applying reward function $r : \cS \times \cA \rightarrow \R$ to the pair of state and action in time-step $t$. Give this pair, the transition kernel $P$ determines the probability of the next state, with $\Prob(S_{t+1} = s_{t+1}|S_t = s_t,A_t = a_t) = P(s_{t+1}|s_t,a_t)$. Starting from the initial state $S_1 \sim \mu$, $T$ rounds of interaction occur.
A memoryless policy is a mapping $\pi : \cS \rightarrow \Delta_\cA$ from state space to distributions over the action space. 

\subsection{Gain, Bias, and Value Function }
For an arbitrary policy $\pi: \cS \rightarrow \Delta_\cA$ and an initial state $s$, associated probability and expected operators are denoted by $\bP_{s}^\pi$ and $\bE^\pi_{s}$. Denote $R^\pi(s) := r(s,\pi(s))$ and $P^\pi(s,s'): = P(s'|s,\pi(s))$ as the reward and transition kernel induced by policy $\pi$. The gain of the policy $g^\pi(s)$ shows the asymptotic average reward obtained by following policy $\pi$ and starting from state $s$, i.e. $g^\pi(s) : = \lim _{T\rightarrow\infty} \frac{1}{T} \bE_s^\pi\sumt R_t $.
The transient part. i.e. $h^\pi(s) : =  \bE_s^\pi [\lim_{T\rightarrow \infty}\sumt(R_t - g(S_t))]$ measures the bias of starting from state $s$, thus $h^\pi(s) \in O(1)$ for all states $s \in \cS$. The Poisson equation states $h^\pi + g^\pi = r^\pi + P^\pi h^\pi$.
The optimal gain is defined as $g^*(s) : = \sup_{\pi \in {\Pi_\cA}} g^\pi(s)$, and the maximizer policy is denoted by $\pi^*$, i.e. optimal policy. For a weakly communicating MDP, $ g^*\in \R \boldsymbol{1}$ where $\boldsymbol{1} $ is the vector full of ones.
It means that the value of the optimal gain does not depend on the initial state which allows us to also use $g^*$ as a scalar when it is clear from context, i.e. $g^*(s) = g^*$ for all $s\in \cS$.
The Bellman operator $L$ operates on a vector $h \in \R^{\bS}$ as $Lh(s): = \max_{a\in\cA}\{r(s,a)+P(.|s,a)h\}$.
It is also known that there exists a solution $h^*$ for the Bellman operator such that $g^* = Lh^* - h^*$, and they satisfy $g^* + h^* \geq g^\pi + h^\pi$ for all $\pi \in \Pi_\cA$.  The value function of a policy captures the cumulative reward obtained by following a policy, i.e. $V^\pi_T(s) := \bE_s^\pi\sumt R_t$. This implies that $V_T^\pi = Tg^\pi + h^\pi + \epsilon$ where $V_T^\pi, g^\pi, h^\pi, \epsilon \in \R^{\bS}$ and $\epsilon \in o(1)$ entry wise. The optimal value starting from the state $s$ is $
    V^*_T(s) : = \sup_{\pi \in \Pi_\cA} V^\pi_T(s)$.
Note that the solution policy for the previous equation can be different for different initial states, unlike the optimal policy with respect to the gain which captures the asymptotic behavior.
The complexity of an MDP can be measured in multiple ways.
\begin{definition} \label{diameter def}
    The \emph{diameter} of an MDP $\cM$ is $ D(\cM): = \sup_{s \neq s'} \inf_\pi \bE_s^\pi[\min\{t > 1 : S_t = s'\}.$ The \emph{span} of a vector $h \in \R^{\bS}$  is $\spn(h) := \max_s h(s) - \min_{s'} h(s').$
\end{definition}
Note that $D$ is only dependent on the transition kernel of the MDP.
Finally, the following classifications will be useful.
\begin{definition}
We say that a MDP is
\emph{Recurrent} (or \emph{ergodic}), if \emph{every} deterministic memoryless policy induces a Markov chain with a single recurrent class;
\emph{Communicating}, if it has a finite diameter (i.e., for every pair of states  $s$ and $s'$ there exists a deterministic memoryless policy under which $s'$ is accessible from $s$ with positive probability);
\emph{Weakly communicating,} if there exists a closed set of states in which each state is accessible from every other using some deterministic policy, plus a possibly empty set of states which is transient under \emph{every} policy.
\end{definition}


\subsection{Average Reward Regret} \label{Bases Algorithms}
The average reward regret of an algorithm $Alg$ is
\begin{equation}
\label{regfixedm}
\sR(Alg,T) := Tg^* - \sumt R_t. 
\end{equation}
Recall that $R_t = r(S_t,A_t)$. Note that in (\ref{regfixedm}) we compare the performance of the algorithm with the asymptotic optimal average reward. One can instead choose $V^*_T(s_1)$, as the benchmark which defines the regret$ \sR_V(Alg, T) = V^*_T(s_1) - \sumt R_t$. The difference between these two notions is $Tg^* - V^*_T(s)$ which is less than $2 \spn(h^*)$ (see Lemma~\ref{gvtgap}). As there is a minimax lower bound for the regret \ref{regfixedm} in $\Omega(\sqrt{\spn(h^*) \bS \bA T})$, this gap is negligible.

The average reward regret captures utility maximization since it compares learner's performance to the highest possible cumulative utilities. This is not the case in the external regret and even in dynamic regret, since they do not consider the underlying adaptivity of the setting. In appendix \ref{othernotions}, we show that average reward regret considers stronger competitors against the learner's performance, in comparison to dynamic regret and episodic regret in online RL.

A recent algorithm proposed by \citet{boone2024achieving}, called Projected Mitigated Extended Value Iteration with doubling trick ($\pmdt$) provides an optimal regret bound based on $\spn(h^*)$. Its power in comparison to previous work, in addition to its tractability, is that it does not need the prior knowledge of $\spn(h^*)$. We use instances of $\pmdt$ as the based algorithm for all $i \in [0: M-1]$, and the performance guarantee $\bb_i(T)$ is obtained by proposition~\ref{baseguarantee}.

$\mathsf{PMEVI}$ runs over the epochs and uses a modified version of Extended Value Iteration (EVI) to compute the policy that the algorithm deploys on each epoch which terminates when the number of visitation of a $(s,a)$ pair doubles (doubling trick \citep{Jaksch2008near}).
$\pmdt$ requires some mild assumption on the confidence regions $\boldsymbol{\cM}_t$, which can be satisfied by empirical estimations of reward and transition kernel, such as the way \cite{Jaksch2008near} generate them. See appendix \ref{properM_tbypsi} or \cite{boone2024achieving} for further explanation. In addition, it needs to know an upper bound $c_{h^*}$ on the $\spn(h^*)$.
\begin{proposition}
\label{baseguarantee}
    (\cite{boone2024achieving}, Theorem 5) Let $c_{h^*}>0$. Assume that $\pmdt$ runs with proper confidence regions $\boldsymbol{\mathcal{M}}_t$, If  $T \geq c_{h^*}^5$, then for every weakly communicating model with $\spn(h^*) \leq c_{h^*}$, $\pmdt$ achieves the following regret with a probability of at least $1-26\delta$,
\begin{align*}
      \sR(\pmdt, T) \in \tilde O(\sqrt{(1+\spn(h^*))\spn(r)\bS \bA T\log(T/\delta)}) 
     + \tilde O(\spn(h^*)\spn(r) \bS^{\frac 5 2} \bA^{\frac 3 2} (1 +c_0)T^{ \frac 1 4})),
\end{align*} where $c_0$ is a universal constant.
\end{proposition}
For proof refer to \citep{boone2024achieving}  (Theorem 5 and equation 13). From\ref{baseguarantee}, we know that there exists a constant $\boldsymbol{C}$ satisfying the following inequality,
\begin{align}
    \label{constt}
    \sR(\pmdt, T) \leq \boldsymbol{C^\frac{1}{3}}\sqrt{(1+\spn(h^*))\spn(r)SAT\log(T/\delta)}.
\end{align}
This shows that the potential regret guarantees for the model class of order $i$ is $\bb_i(T,\delta) = C_i\sqrt{T \log (T/\delta)}$, where $C_i = \boldsymbol{C}^\frac{1}{3}\sqrt{(1+\spn(h_i^*))\spn(r_i)\bS_i\bA_i }$.




\subsection{Model Selection Problem}
Let $\cC_0 \dots \cC_{M-1}$ be $M$ different model classes.  For each order $i \in \{0,\dots,M-1\}$, $\cC_i$ consists of a set of MDPs with a common state space $\cS_i$, action space $\cA_i$, and horizon $T$, but with differing transition kernel and reward functions. Each module class $\cC_i$ has a base algorithm $\balg_i$ with a performance guarantee $\bb_i(T,\delta)$, which is a high probability upper bound over the regret of using $\balg_i$ in any MDP $\cM \in \cC_i$.
\begin{definition} \label{well/miss specified}
    (Misspecified and well-specified base algorithms) A base algorithm $\balg_i$ is \emph{well-specified} when the regret of its interaction with the environment for $T$ rounds is upper bounded by $\bb_i(T)$; otherwise it is \emph{misspecified.}
\end{definition}
We say an MDP represents an interaction with the environment if the distributions over the next states and the rewards obey $P$ and r. Thus if a model class $\cC_i$ is not rich enough to contain an MDP that represents the interaction with the environment, then the regret $\balg_i$ might not be able to satisfy its regret bound $\bb_i(T)$, and will be misspecified.

\begin{assumption} \label{Nested Assumption}
We assume that if $\balg_i$ is a well specified base algorithm, then $\balg_j$ is also well-specified for all $j > i$.   
\end{assumption}

\begin{assumption} \label{Realizability Assumption}
(Realizability Assumption) We assume that there exists an order $m\in [0:M-1]$ such that $\balg_m$ is well specified. The smallest such $m$ is called the \emph{optimal order} and denoted by $m^*$. In addition we assume $
g^*_{m^*} =  g^*_{m^*+1} = \dots = g^*_{M-1}$, 
where $g^*_i$ for $i \in [m^*:M-1]$, is the optimal gain of the representing MDP $\cM_i \in \cC_i$.
\end{assumption}

This equalities among the optimal gains is not an artificial assumption. It states that if we are in the optimal order $m^*$,  we will not get a larger optimal gain by going to the larger orders. To have a well-defined regret this assumption is needed.

\parhead{Regret.}
The notion of regret (\ref{regfixedm}) is designed for a fixed order $m$. In model selection for average reward, as $m^*$ is unknown the performance of the algorithm is measured by,
\begin{align}
    \label{Reg}
    \Reg(\mathsf{Alg}, T):= T g^*_{m^*} - \sumt R_t. 
\end{align}
The goal of an online model selection algorithm is to achieve a sublinear regret \eqref{Reg}. This intuitively means, learning the optimal model class and an optimal policy within it simultaneously.  

\section{Algorithm } \label{algorithms}
We propose an online model selection algorithm called \textbf{M}ultiplicative \textbf{R}egret \textbf{B}alancing and \textbf{E}limination in \textbf{A}verage \textbf{R}eward ($\mathsf{MRBEAR}$). The meta algorithm $\bear$ is designed based on regret balancing and elimination technique and a misspecification test proper for average reward setting. Figure~\ref{MRBEAR} shows the pseudo-code of $\mathsf{MRBEAR}$. We consider all the model classes with memory order $m < M$ (Refer to \ref{Realizability Assumption}). At the beginning, the algorithm runs each base algorithm $\balg_i$, for $c_{h^*}^5$ iterations, where $c_{h^*}$ is an upper bound on the $\spn(h_i^*)$ for all $i \in [m^*,M-1]$. This warm-up phase is to bring all well-specified base algorithms in the region that their guarantee holds. After that, we start with the initial set of active model classes $\cI_0 =\{ 0,1,\dots, M-1\}$. $\bear$ proceeds in epochs, in each epoch $k$, first, it updates the set of active model classes by checking the misspecification test  \eqref{misspecification test} over all of the classes. Then it takes the model class with the smallest $\bb_i$, $i_k$, and runs the base algorithm $\balg_{i_k}$ on $\cC_{i_k}$ until it terminates due to doubling trick after $n_k$ iterations \citep{Jaksch2008near}.\footnote{We use $\pmdt$ as the base algorithm, which efficiently achieves the optimal regret guarantee of $\tilde O(\sqrt{\spn(h^*) SAT \log(T)})$. However, $\bear$ as a meta-algorithm admits any other base algorithm which is based on the doubling trick and has a potential regret guarantee in the form of $\bb_i(T,\delta) = C_i\sqrt{T \log(T/\delta)}$.}
Then the algorithm updates the set of iterations consumed on each model class $i$, $\cN_{i,k}$ which has the cardinality of $ |\cN_{i,k}| = N_{i,k}$. $\bear$ terminates by spending $T$ iterations. We analyze the performance of $\bear$ in section \ref{RegAnalysis}.
\begin{algorithm}
\caption{Multiplicative Regret Balancing and Elimination in Average Reward (MRBEAR).}\label{MRBEAR}
\begin{algorithmic}
\Require $\delta,M,T,c_{h^*},\balg_{0:M-1}, \bb_{0:M-1}$
\State $\cI_0 = \{ 0,1,\dots,M-1\} \ , \ t_0 = 0$
\For{$i \in \{0,\dots,M-1 \} $} \Comment{Warm up phase}
\State Run $\balg_i(\delta, c_{h^*})$ for $\max\{\cfive,9\}$ iterations
\State $N_{i,0} = \cfive \ , \ \cN_{i,0} = \{ t_0+1 ,\dots, t_0 + \cfive \}$
\State $t_0 = t_0 + \cfive$
\EndFor
\For{$k \in \{1,2,\dots\}$}
\For{$i \in \{ 0,\dots,M-1\}$}
\If{\Comment{Checking misspecification test} $$\frac{\bb_i(N_{i,k-1},\delta) + \sum_{t \in \cN_{i,k-1}} r_t}{N_{i,k-1}} < \max_{j \geq i} \frac{\sum_{t \in \cN_{j,k-1}} r_t - 2 c_{h^*}}{N_{j,k-1}}$$} 
\State $\cI_{k} = \cI_{k-1} \backslash \{ i\}$ \Comment{Updating Active Classes}
\EndIf
\EndFor
\If{$t_{k-1} \geq T$} break
\State Pick $i_k = \arg \min_{i \in \cI_k} \bb_i(N_{i,k-1})$
\State Run an inner epoch of $\balg_{i_k}$ on $\cC_{i_k}$ ($n_{i_k}$ iterations) \Comment{An episode of $\pmdt$}
\State Update parameters:
\State $t_k = t_{k-1} + n_{i_k} \ , \ \cN_{i_k, k} = \cN_{i_k, k-1} \cup \{t_{k-1} +1,\dots t_k \}$
\State $N_{i_k, k} = N_{i_k, k-1} + n_{i_k} \ ,\ \forall j \neq i_k : N_{j, k} = N_{j, k-1} $
\EndIf
\EndFor
\end{algorithmic}
\end{algorithm}

\section{Regret Analysis}
\label{RegAnalysis}
In each time step, regret balancing and elimination algorithm maintains a set of active model classes that contain the optimal model class with high probability \citep{abbasi2020balancebanditrl,agarwal2017corralling}.
At each time step, the algorithm picks the active model class with the smallest regret guarantee; thus, all active model classes will have regrets in the same order.  In the average regret setting, we must commit to a model class for multiple iterations; thus, we check for model misspecification less often than every time step.  We will show that even under this less-frequent checking, the active set will contain all of the well-specified model classes, and more importantly their regret can be maintained in the same order.  However, in contrast to work in the episodic setting, $\bear$ maintains a multiplicative balance over the regrets rather than additive \citep{pacchiano2020regretbalanceelim}.

\begin{theorem}[Main theorem]
\label{maintheorem}
By running the algorithm $\bear$ with base algorithms of $\pmdt$, over $M$ model classes $\cC_0$ to $\cC_{M-1}$, and the unknown optimal model class $\cC_{m^*}$, and known upper bound of $c_{h^*} \geq \spn(h_i^*)$, for $T \geq M c_{h^*}^5$ and all $0< \delta < 1$, with probability of at least $1- 26MT \delta$, the regret \eqref{Reg} is upper bounded by
\begin{align*}
\Reg(\bear , T)  \leq \left(\frac{m^* C_{m^*}^2 \log^{\frac{3}{2}}(T /\delta)}{(1-\alpha)^2 C_0^2 } + \frac{M}{1-\alpha} \right)\bb_{m^*}(T,\delta) + O(\log^\frac{3}{2}(T/\delta))
\end{align*}
where $\frac {1}{2} \leq \alpha =\frac{\log(c_{h^*}^5 \lor 9)+1}{2\log(c_{h^*}^5 \lor 9)} \leq \frac{3}{4}$.
\end{theorem}
Before explaining the proof, let us compare the regret bound of \ref{maintheorem} with a trivial approach. Due to the realizability assumption \ref{Realizability Assumption}, one can run an instance of $\pmdt$ directly on the model class $M-1$. That would give the regret bound of $\bb_{M-1}(T,\delta) = C_{M-1} \sqrt{T \log(T/\delta)}$ (Refer to \ref{baseguarantee}). In many cases, the growth of $C_i$ is exponential in $i$ ( corollary \ref{corollary}). Therefore, the result in theorem \ref{maintheorem}, i.e., $\tilde O(M C_{m^*}^2 \bb_{m^*}(T,\delta))$ can be notably better than $\bb_{M-1}(T,\delta)$. Now we present the ideas in the proof of the main theorem \ref{maintheorem}.

\parhead{Regret Decomposition.}
We decompose the regret into the regret of each model class as follows,
\begin{align*}
    \Reg(\bear, T) := T g^*_{m^*} - \sumt R_t  = \sum_{i = 0}^{M-1} \underbrace{\big [ N_{i,\K(T)} \ g^\star - \sum_{t \in \cN_{i,\K(T)}} R_t \big ]}_{\Reg_i(\balg_i, N_{i,\K(T)})}.
\end{align*}
where $\K(t)$ is the number of epochs until iteration $t$. The important note is that most of the average reward algorithms, including $\pmdt$ algorithm that we use, do not count on the transitions between epochs. In other words, after ending an epoch due to the doubling trick, they do not pay attention to or exploit that the first state of the next epoch is a successor of the last state, action pair of the previous epoch. This allows us to pause the base algorithm $\balg_i$ on model class $\cC_i$ after it terminates its epoch, then spending a couple of epochs in the other model classes, and then get back to $\cC_i$ and resume the base algorithm $\balg_i$. Therefore the use of $t \in \cN_{i,K(T)}$ in above expression is valid. Also, this changing of the model classes does not harm the regret guarantee of the well-specified model classes. 

\parhead{High probability event and misspecification test.}
Consider the following event,
\begin{align*}
        \cG = \{  \forall i\in [m^*:M-1] ,  k \in [1:\K(T)+1] :\Reg_i(\pmdt,N_{i,k}) \leq C_i\sqrt{ N_{i,k} \log(\frac{N_{i,k}}{\delta})} \}
\end{align*}

which captures the event that all the well-specified model classes, at the beginning of all epochs, satisfy their regret guarantee. $\cG$ happens with probability of at least $1- 26\delta (M- m^*)(\K(T)+1)$. The proof is based on proposition~\ref{baseguarantee} and union bound (see appendix \ref{Gishighprob}). Under the event $\cG$ we know that if $i \in [m^* : M-1]$ is well specified for all $j \geq i$, $g^*_i = g^*_j = g^\star$. Also for all $k\in [1:\K(T)+1]$ we have,
\[
- 2 \spn(h^*_i)\leq \sR_i(\balg_i, N_{i,k}) \leq \bb_i( N_{i,k}, \delta),
\]
where the lower bound is a well-known result (Refer to \ref{gvtgap}). Thus by noting that $c_{h^*} \geq \spn(h^*)$, we can write,
\[
\frac{\bb_i(N_{i,k}, \delta) + \sum_{t \in \cN_{i,k}} R_t }{N_{i,k}} \geq g^*_i = g^*_j \geq \frac{\sum_{t \in \cN_{j,k}} R_t  - 2c_{h^*}}{N_{j,k}}.
\]
This expression can be used as a misspecification test. Although $g^*_i$ and $g^*_j$ are unknown, the reward that the algorithm gathers in each model class and the regret bound $\bb_i(N_{i,k}, \delta)$ are observable and computable. Therefore, violation of the inequality
\begin{align}
\label{misspecification test}
     \frac{\bb_i(N_{i,k}, \delta) + \sum_{t \in \cN_{i,k}} r_t }{N_{i,k}} \geq \max_{j \geq i} \frac{\sum_{t \in \cN_{j,k}} r_t  - 2c_{h^*}}{N_{j,k}}
\end{align}
by an $i \in [0: M-1]$, indicates that \emph{at least} $\balg_i$ is not well specified and $i$ should be eliminated (or $\cG$ has not happened which is low probability). This also means that with high probability all the well-specified model classes remain active in all epochs. Thus we use the above inequality as the misspecification test in $\bear$.

\parhead{Regret Balancing.} The following is a key lemma.
\begin{lemma}
\label{Regretisbalanced main text}
By running algorithm $\bear$ with base algorithms of $\pmdt$, for all $0< \delta < 1$, $k \in \N$ and any pair of $i ,j \in \cI_k $ where $i \neq j$, and when $\bb_i(N_{i,k},\delta) = C_i\sqrt{N_{i,k} \log(\frac{N_{i,k}}{\delta})}$, the following hold.
\begin{enumerate}
    \item $\bb_i(N_{i,k}, \delta) \leq \bb_j(N_{j,k},\delta) + \alpha \bb_i(N_{i,k-1},\delta)+ \beta$
\item
    \label{NoverN}
$\frac{N_{i,k}}{N_{j,k}} \leq$ $\left(\frac{C_j}{(1-\alpha)C_i} + \frac{\beta}{(1-\alpha)C_i\sqrt{N_{j,k} \log(N_{j,k}/\delta)}}\right)^2 \log(\frac{N_{j,k}}{\delta})$,
\end{enumerate}
where $\alpha = \frac{\log(c_{h^*}^5 \lor 9)+1}{2\log(c_{h^*}^5 \lor 9)}$  and  $\beta = \bb_{M-1}(c_{h^*}^5,\delta) -\bb_0(c_{h^*}^5,\delta) $.
\end{lemma}
The proof is in appendix \ref{Regretisbalanced} by induction on $k$. It exploits the doubling trick that $\pmdt$ uses, and the definition of $\bb_i$. From this lemma, we reach the following bound on the regret of misspecified model classes, based on the number of iterations consumed on them and on $\cC_{m^*}$.
\begin{lemma}
    \label{reg_iboundwithN main text}
    For any active model class $i \in \cI_k$  such that $i<m^*$, under the event $\cG$, the regret of model class $i$ is bounded as follows,
    \begin{align*}
      \Reg_i(N_{i,k-1})
\leq  (\frac{N_{i,k-1}}{N_{m^*,k-1}} + \frac{1}{1-\alpha})\bb_{m^*}(N_{m^*,k-1}) +\frac{2 N_{i,k-1}}{N_{m^*,k-1}} c_{h^*} + \frac{\beta}{1-\alpha}.
    \end{align*}
\end{lemma}
The proof can be found in appendix \ref{reg_iboundwithN}. Again from lemma~\ref{Regretisbalanced main text} we know that the fraction $\frac{N_{i,k-1}}{N_{m^*,k-1}}$ is controlled which gives us the following result.
\begin{lemma}
\label{finalReg_ibound main text}
For any active model class $i \in \cI_{\K(T)+1}$  such that $i<m^*$, under the event $\cG$, the regret of model class $i$ is bounded as follows.
\begin{align*}
    \Reg_i(N_{i,\K(T)}) \leq \frac{C_{m^*}^3 \sqrt{N_{m^*,\K(T)}} \log^2(N_{m^*,\K(T)} /\delta)}{(1-\alpha)^2 C_i^2 } +   \frac{1}{1-\alpha}\bb_{m^*}(N_{m^*,\K(T)})+ O(\log^\frac{3}{2}(T/\delta))
\end{align*}
\end{lemma}
The factor $ \frac{1}{2}\leq \alpha \leq \frac{3}{4}$ in lemma~\ref{Regretisbalanced main text} makes a multiplicative balance between the guarantees, which roughly means that for all $j \geq m^*$ the guarantee $\bb_j \leq \frac{1}{1-\alpha} \bb_{m^*}$. Note these are well-specified regret guarantees so that they can really be an upperbound on the performance of $\balg_j$. Thus, we have $\sum_{i = m^*}^{M-1} \Reg_i \leq \frac{M-m^*}{1-\alpha} \bb_{m^*} $. For the other term corresponded to misspecified model classes, i.e.,  $\sum_{i = 0}^{m^*-1} \Reg_i $, we use lemma~\ref{finalReg_ibound main text}. Adding these two parts implies theorem~\ref{maintheorem} (see appendix~\ref{maintheorem proof}).

\section{Application to Repeated Games}
\label{application to games}

In this section, we come back to our motivating question, and apply the model selection result to repeated games.  The learner's goal is to maximize its utility without knowledge of opponent's memory order. Consider a two-player game $G$ between a learner $\mathbb{A}$ and an opponent $\mathbb{B}$. The action spaces of the two players are denoted by $\cA$ and $\cB$ respectively, and we denote their cardinality by $\bA = |\cA|$ and $\bB = |\cB|$. Both player's are going to simultaneously play the stage game $G$, repeatedly for $T$ rounds. Learner's utility function denoted by $\bU : \cA \times \cB \rightarrow [0,1]$, given an action profile outputs a positive real value less than or equal to $1$.
We denote the two players' actions for time step $t$ by $A_t$ and $B_t$. The pair of $A_t$ and $B_t$ makes the interaction of time step $t$ denoted by $O_t = (A_t,B_t)$. We save the history of the game in a tuple $\cO_t = ((A_1,B_1),(A_2,B_2),\dots,(A_t,B_t)) = (O_1,O_2,\dots,O_t)$.
A (not necessarily memroy-less) policy $\boldsymbol{\pi}$ is a tuple of $(\pi_1,\dots,\pi_T)$, where each $\pi_t: (\mathcal{A}\times \mathcal{B})^{t-1} \rightarrow \Delta_\cA$ gets the history of the game and outputs a distribution over the  learner's action space. The resulting distribution gives the mixed strategy played by the learner in the stage game at that time step.
\begin{definition} \label{policyorder}
    A policy $\boldsymbol{\pi} = (\pi_1,\dots,\pi_T)$ is $m$-th order if $m \in \N$ is the \emph{smallest} natural number such that there exists a $\pi : (\cA \times \cB)^m \rightarrow \Delta_\cA$ where,
    \[
    \pi_t(o_{1:t-1}) = \pi(o_{t-m:t-1}) \quad \forall t \in [T]\ ,\ o_{1:t-1} \in (\cA \times \cB)^{t-1}.
    \]
In other words, the output distributions of $\boldsymbol{\pi}$ only depend on the last $m$ interactions in history, and are independent from the previous ones, and $t$. In this case $\boldsymbol{\pi}$ is fully representable by $\pi$. 
\end{definition}
With slight abuse of notation, we denote $ \Prob_\pi[A_t=a_t|O_{t-1:t-m} = o_{t-1:t-m}] = \pi(a_t|o_{t-1:t-m})$. 
The set of all $m$-th order policies over action space $\cA$ is denoted by $\Pi_\cA^m$. The set of all policies over $\cA$ is given by $\Pi_\cA = \bigcup_{m=0}^\infty \Pi_\cA^m$.

\parhead{Best Response.} The best response (in the repeated game) against policy $\boldsymbol{\psi} \in \Pi_\cB$ is computed as follows,
\begin{align} \label{BR}
    \mathsf{BR}(\boldsymbol{\psi},s_1)  = \arg\max_{\boldsymbol{\pi} \in \Pi_\cA} \E \sumt \bU(A_t,B_t) = \arg\max_{\boldsymbol{\pi} \in \Pi_\cA} V^{\boldsymbol{\psi},\boldsymbol{\pi}}_T(s_1),
\end{align}
where $A_t \sim \pi_t(.|o_{1:t-1}) $and $B_t \sim \psi_t(.|o_{1:t-1})$.

\begin{assumption}
    We assume $\boldsymbol{\psi}$ is an $m^*$-th order policy where $m^*$ is unknown to the learner, while the learner knows of an upper bound $M$ such that, $M > m^*$. 

\end{assumption}

\label{RL Setting}
\parhead{Modeling as Average Reward RL.}
The Learner considers a model class $\cC_m$ for each possible memory order of $0\leq m \leq M-1 $. Suppose $m\geq m^*$. All of the MDPs in $\cC_m$ share the state space $\cS_m = (\cA\times \cB)^m$ with the cardinality $|\cS_m| = \bS_m$, and the action space $\cA$, which is the set of learner's pure strategies in $G$.

Given a state $s_t = (o_{t-1},\dots, o_{t-m})$ and action $a_t$, the reward $R^\psi_t = r(s_t,a_t) = \bU(a_t,B_t)$ is obtained by applying learner's utility function to the action profile played in time-step $t$. As $\psi$ is unknown, the action $B_t$ is random, and the distribution over the rewards is also unknown. The transition dynamic induced by $\psi$, $P^\psi : \cS_m \times \cA \rightarrow \cS_m$, is defined as follows,
\begin{align*}
     P^\psi(s' | a,s)  = \Prob[o'_{m:1} | a, o_{m:1}]
 = \begin{cases}
    \psi (b'_m|o_{m:m-m^*+1}) & \text{if}\  a'_m = a \ , \ o'_{m-i} = o_{m-i+1} \\
    0 & \text{otherwise }
\end{cases}
\end{align*}

where $o_i=(a_i,b_i)$ and $o'_i=(a'_i,b'_i)$ for all $i \in [m]$.
\begin{assumption}
\label{weakly communicating assumption}
The opponent's policy $\psi$ induces a weakly communicating MDP.
\end{assumption}
For $m$-th order policies, the very first $m$ actions can not be taken as the history is not long enough to make the first state. We assume the first $m$ actions, (for simplicity indexed by $a_{1-m:0}$ in negative time steps $t\in[1-m:0]$) are chosen such that the initial state $s_1$ is a sample from distribution $\mu$. As we assume $\psi$ induces a weakly communicating  MDP, the initial state does not affect the performance of the algorithm.
These elements collectively build the MDP of $\cM=(\cS_m,\cA,P^\psi,R^\psi,T,\mu)$. We omit superscript $\psi$ on $P^\psi$ and $R^\psi$ when having no change in $\psi$ is clear from the context.


\begin{remark} \label{m<m^* models remark}
It is important to note that only for $m \geq m^*$ the opponent's policy $\psi$ induces a well-defined \emph{Markov} transition probability and reward function. The observed actions of player $\pB$ coming from an $m^*$-th order $\psi$, and their resulting rewards, are not representable with any lower order model classes $m < m^*$ , thus the regret guarantees of those under specified based algorithms may not hold. The corresponding base algorithms for $m \geq m^*$ are all well-specified.
\end{remark}



As it was clear from the best response definition (\ref{BR}), the value of a policy $\pi$ is $V^{\psi,\pi}_T(s_1) = \E_{\pi,\psi} \sumt \bU(A_t,B_t)$, and is maximized by the best response policy $\mathsf{BR}(\psi,s_1)$. However, the \emph{optimal policy} $\pi^*_\psi$ is the one that maximizes $g^\pi_\psi = \lim_{T\rightarrow \infty} \frac{1}{T} V^{\psi,\pi}_T(s_1)$, and is independent of the initial state $s_1$ in weakly communicating MDPs (Refer to assumption \ref{weakly communicating assumption}). In the appendix \ref{gvtgap} we show the difference between the rewards obtained by $\mathsf{BR}(\psi,s_1)$ and $\pi^*_\psi$ is less than $2 \spn(h^*)$, which is negligible.

\begin{proposition} {\label{prop0}}
     Against any $m$-th order policy $\psi$ and for any policy $\pi \in \Pi_\cA$, there exists a policy $\pi'\in \bigcup_{i=0}^m \Pi_\cA^i$ in the order of at most $m$, where $\pi$ and $\pi'$ have the same value, i.e., $V^{\psi,\pi}_T(s_1) = V^{\psi,\pi'}_T(s_1)$ for all $s_1 \in \cS_m$ and $T \in \N$.
\end{proposition}
The complete proof is in appendix \ref{prop0 proof}.
This implies that the domain of the maximum in (\ref{BR}) can be $\Pi^{m^*}_\cA$ instead of $\Pi_\cA$.

\subsection{Regret Upper Bound in Game Setting}
The opponent by taking a policy $\psi$ in the order of $m^*$ makes all $\balg_{m^*} \dots \balg_{M-1}$ well specified. Therefore the assumption \ref{Nested Assumption} is satisfied. Also, proposition \ref{prop0} implies that the optimal gain $g^*_\psi$ is invariant in the induced MDPs of $\psi$ in $\cC_{m^*}$ to $\cC_{M-1}$. Therefore the assumption \ref{Realizability Assumption} holds as well. Another important result from proposition \ref{prop0} is that in all well-specified model classes, $i \in [m^*,M-1]$ the spans of $h^*_i$ are the same. Thus we can use $\spn(h^*)$ without specifying the model class in which optimal bias is defined. Refer to appendix  \ref{prop0 proof} for the proof of previous claims. Also, the mild assumptions needed for the base algorithm $\pmdt$ are in the appendix \ref{properM_tbypsi}. Thus we can apply $\bear$ to the repeated game setting.

Depending on $\psi$'s input, we focus on two types of opponents. First, general opponent for which we only assume an unknown limited memory of $m^*$ for the opponent, and $\psi$ might depend on both the learner's and opponent's actions, i.e., $\psi: (\cA \times \cB)^{m^*} \rightarrow \Delta_\cB$. Second, self-oblivious opponent, which is the case that $\psi: \cA^{m^*} \rightarrow \Delta_\cB$  only depends on the learner's actions. A self-oblivious opponent induces a known deterministic transition kernel. Further explanation about different opponents is in appendix \ref{Different Types of Opponent}.



\begin{corollary}
\label{corollary}
For $T$ rounds of playing a repeated game, against an opponent with an unknown memory $m^* \leq M$,  using $\bear$ with $\pmdt$ as base algorithm, gives us the following bound on the regret with a probability of at least $1 - 26MT\delta$,
\begin{align*}
\Reg(T) \leq \boldsymbol{C} M (\bA^{m^*+1} \bB^{m^*} \spn(h^*)\log(T /\delta))^\frac{3}{2} \ T^\frac{1}{2} 
\end{align*}
where $\boldsymbol{C}$ is a constant satisfying \ref{constt}. And for self-oblivious opponents we have,
\begin{align*}
    \Reg(T) \leq \boldsymbol{C} M (\bA^{m^*+1} \spn(h^*))^\frac{3}{2} \ T^\frac{1}{2} \log^{\frac{3}{2}}(T /\delta)
\end{align*}
\end{corollary}
This corollary is the result of the main theorem \ref{maintheorem} in the repeated game setting. One can easily show that the choice of $\delta  = \frac{\log(T)}{\sqrt{T}}$ gives an expected regret in $O(M (\bA^{m^*+1} \bB^{m^*} \spn(h^*))^\frac{3}{2} \ T^\frac{1}{2})$.

In our setting, the size of the state space (i.e., $(\cA \times \cB)^m$) increases exponentially with respect to the memory $m$. Therefore, instead of having a regret bound in $\tilde  O(\sqrt{\bA^M \bB^{M-1} T})$ by directly learning in $\cC_M$ , $\bear$ enjoys from a regret in $\tilde  O(M\sqrt{\bA^{3(m^*+1)} \bB^{3m^*} T})$ which is an exponential improvement with respect to the memory. Furthermore, in lowerbound section \ref{lowerbound}, we show that having the exponent of $m^*$, the unknown opponent's memory, in regret bound is inevitable.

\subsection{From $\psi$ to $\spn(h^*)$} \label{from psi to span}

In this section, we study how different choices of $\psi$ with the same memory of $m$ affect the span of optimal bias and give an upper bound $c_{h^*}$ on $\spn(h^*)$ based on the Kemeny's index induced by $\psi$. Both propositions are proved in appendix\ \ref{proof of upperbound on span}.

\begin{proposition}
    \label{span-psi prop}
    For any policy $\pi$ in an ergodic or unichain MDP we have $\spn(h^\pi) \leq 2 \spn(r^\pi)\kappa^\pi$, where $h^\pi$ and $r^\pi$ are respectively the bias and reward vectors induced by $\pi$ and $\kappa^\pi$ is the Kemeny's constant of the Markov chain with transition kernel $P^\pi$. The same inequality holds in weakly communicating MDPs for the optimal policy $\pi^*$,
        $\spn(h^*) \leq 2 \spn(r^*)\kappa^*$.
\end{proposition}
In the repeated game setting, as the range of the utility function is assumed to be $[0,1]$, $\spn(r^\pi) \leq 1$ for all $\pi$. 
In the self-oblivious case, by leveraging on the structure of the MDP, we can have the following upper bound on the $\spn(h^*)$,
\begin{proposition}
\label{self-oblivious span}
    In a self-oblivious model $\cM$ with the order of $m$ (described in appendix \ref{deter. tran.}), for every policy $\pi$ the upper bound $ D(\cM) \leq m $ on the diameter holds. This implies $\spn(h^*) \leq  m \spn(r^*)$.
\end{proposition}

\subsection{Lower Bound}
\label{lowerbound}
The MDPs constructed in the repeated game setting have a special structure that prevents us from directly using previous lower bounds designed for generic MDPs  \citep{Jaksch2008near,osband2016lowerbound}. In this section, we present a minimax lower bound on the regret of any algorithm, using Le Cam method. First, we give a divergence decomposition lemma (appendix \ref{proof of lowerbound}), which is similar to the the key lemma in Le Cam-based lower bounds in bandit literature. Then, using de Bruijn sequences \citep{de1975acknowledgement, crochemore2021problems125}, we design two complementing opponent policies $\psi$ and $\psi'$ forcing the algorithm to have high regret in at least one of them. 
\begin{theorem}
    Suppose the number of opponent's actions $|\cB |\geq 3$ and number of learners actions $|\cA| \geq 2$. For any fixed memory $m \in \N$ known for the learner, and any algorithm $Alg$, there exist a stage game with utility $U : \cA \times \cB \rightarrow [0,1]$, and a general opponent's policies $\psi_{Gen}$ such that
   \[
\sR(Alg,T) \in \Omega (\frac{1}{m} \sqrt{\bA^{m-1}\bB^{m-1} T }).
\]
\end{theorem}
The proof of the previous lemma and theorem is in appendix \ref{proof of lowerbound}.
This shows that the exponential dependency to the memory $m^*$ in theorem \ref{maintheorem} is inevitable.

\section{Conclusion} \label{conclusions}
We propose $\bear$ as an online model selection algorithm for the average reward setting. 
We show that it enjoys a regret of $\Ot (M C_{m^*}\bb_{m^*}{T,\delta})$, where $M$ is the number of model classes and $C_{m^*}\bb_{m^*}(T, \delta)$ is the regret guarantee for the optimal well-specified model class.
We construct a framework for solving repeated games when the learner is facing an opponent with unknown limited memory $m^*$ and obtain a regret bound in $\Ot (M\sqrt{\bA^{3(m^* +1)}\bB^{3 m^*}T})$ for the learner by using $\bear$. We show that the exponential dependency in $m^*$ is unavoidable.

\parhead{Future Work.}
On the model selection side, it is interesting to study whether the dependence on $M$ can be made additive, or to prove a lower bound on the regret of model selection.  Extending \bear{} to different base algorithms, or to an any-time algorithm are also attractive directions.  It may also be possible to relax the weakly communicating assumption to multi-chain.
The framework for repeated games we define could be extended in a number of ways, e.g., by considering repeated Stackelberg games or Markov games. Analyzing the equilibrium behavior of \bear{} when used in self-play is a particularly promising avenue.

\clearpage
\bibliographystyle{apalike}
\bibliography{bib}

\newpage
\appendix

\section{Extensive Related Work} \label{extensive related works}
\parhead{Average Reward.}
\citet{auer2008near} achieve a regret bound of $O(DS\sqrt{AT})$ with their seminal algorithm \textsf{UCRL2}.
Further work improves the dependency on the diameter $D$ by considering the number of next probable states or local diameters \citep{fruit2020Bern,bourel2020localDiam}. Other work uses mixing time or hitting time of the Markov chain induced
by optimal policy \citep{wei2020model,ortner2020regret}, or by using the variance of the next state to have problem-dependent regret bounds \citep{talebi2018variance}. Diameter based algorithms usually assume communicating MDPs. However, \citet{fruit2018TUCRL} propose a UCRL-based algorithm for weakly communicating MDPs, by concatenating UCRL2 with an estimator of communicating part of the MDP. Another measure used to specified MDP complexity is the span of optimal bias $\spn(h^*)$. When rewards are bounded, it is not hard to show that $\spn(h^*) \leq D$.
There even exist weakly communicating MDPs with infinite diameter, but with finite $\spn(h^*)$.
These show that results that capture $\spn(h^*)$ tend to be stronger \citep{fruit2018SCAL,zhang2019regret,boone2024achieving}.
The leading algorithm is the very recently introduced PMEVI, which efficiently achieves the optimal regret bound of $\tilde O(\sqrt{\spn(h^*)} SAT)$ without prior knowledge of $\spn(h^*)$ \citep{boone2024achieving}.

\parhead{Model Selection.}
There is a rich literature on model selection for statistical learning \citep{vapnik2006estimation,lugosi1999adaptive,massart2007concandmodel,bartlett2002model} and online learning \citep{freund1997decision,foster2017parameter-free}. For the case of contextual bandit and reinforcement learning, many approaches rely on the same underlying idea: using a meta-algorithm above all model classes, to determine which of the model classes are well specified based on interaction with the environment. This can be accomplished by online misspecification tests, like checking whether the empirical regret satisfies known bounds \citep{abbasi2020balancebanditrl,pacchiano2020regretbalanceelim}.
Some work assumes nested structure on the model classes \citep{ghosh2021modelgenericRL,foster2019model}, but some do not \citep{agarwal2017corralling}. The cost of model selection under different assumptions can be either multiplicative \citep{pacchiano2020regretbalanceelim} or additive \citep{ghosh2021modelgenericRL} to the regret of the best well-specified model class \citep{marinov2021pareto,krishnamurthy2024costlessmodelselectio}. All of the above papers analyze episodic regret; to the best of our knowledge, there is no work in model selection for the average reward setting.
Our results are a first step in extending the model selection approach to the average reward setting.

\parhead{Learning in Markov Games.}
The other related direction to this work is learning in stochastic games \citep{wei2017online,zhong2021can,xie2020learningSG}. One challenge for learning in Markov games lies in the trade-off between learning a Markov (C)CE (corresponding to time-independent no-regret policies) and exponential dependency on the number of agents \citep{daskalakis2023complexity,jin2021v}.  An attractive recent line of research studies regret minimization when the opponent runs online learning algorithms\citep{brown2024learning, deng2019strategizing,braverman2018selling}. The recent paper of \citet{assos2024maximizing} shows utility maximization algorithm against opponents deploying multiplicative weights updates (MWU) algorithm is tractable in zero-sum repeated games, while there is no Fully Polynomial Time Approximation Scheme (FPTAS) for utility maximization against a best-responding opponent in general sum games. Note that in this work we assume limited memory for the opponent which excepts us from the previous impossibility result. Another very recent work is \citep{nguyen2024learning} in which a bounded memory for the opponent is assumed and the performance of the learner is measured by policy regret. However \citet{nguyen2024learning} model the interaction as a Stackelberg game, in episodic Markov Games. We study average reward and simultaneous games which are arguably more complicated in this work. \citet{zhong2021can} use reinforcement learning to solve the Stackelberg equilibrium for myopic followers. \citet{wei} study stochastic games under average reward using the diameter as the complexity measure, while this work is based on $\spn(h^*)$ for the repeated game settings. In addition, we use different model classes. \cite{xie2020learningSG} by assuming linearly representable transition kernel and reward functions, apply function approximation to learn course correlated equilibria in Markov game. The setting in this paper differs from the previous work either in the notion of regret, or assuming different levels of memory for the opponent, or the measure of complexity. 







\section{Additional Explanations}
\subsection{Connection to Other Notions of Regret} 
\label{othernotions}

In this section, we compare average reward regret to the common notions of regret like dynamic regret and episodic regret in RL (which is an external notion of regret).

\parhead{Episodic Regret.}
Episodic regret is a common notion in online reinforcement learning. Suppose there are $K$ episodes, each of $H$ iterations' length. The episodic regret is defined as
   \begin{align*}
    \sR_E(Alg,\psi)  & = \sum_{k=1}^K V_H^{\pi^*}(s_1) - V_H^{\pi_k}(s_1),
\end{align*}
where $\pi_k$ is the fixed policy employed by the algorithm in episode $k \in [K]$. This can be a natural evaluation when there is an episodic structure in the environment.  The regret $\sR_V$ can be interpreted as a single episode interaction plus the ability to change the policy in each interaction. 
The point that makes this regret easier to deal, in comparison to average reward regret, is that one can rewrite episodic regret as the sum of $K$ practically (and not statistically) independent immediate regrets $\Delta_k = V_H^{\pi^*}(s_1) - V_H^{\pi_k}(s_1)$, which is common in external notions of regret. This can not be done in the average reward regret, as one action might have a poor immediate reward but causes notably better future rewards, and/or vice versa.

\parhead{Dynamic Regret.}
Suppose an episodic regret with a planning horizon of $H =1$ (i.e., $K = T$), where the initial state of the episode $k>1$ is the state that the transition dynamic suggests based on $s_{k-1}$ and $a_{k-1}$. In this case, you can change the policy at every iteration. Furthermore, $V^{\pi^*}_H (s_k)$ coincides with the highest possible immediate reward. This gives us the \emph{adaptive} dynamic regret,
\begin{align*}
    \sR_{AD}(Alg,\psi) & = \E \big [ \sumk V_1^{\pi^*}(S_k) - V_1^{\pi_k}(S_k) \big ]\\
    & =  \E \big[ \sumt
\max_{a \in \cA} r( a, S_t) - r(A_t , S_t)  \big]
\end{align*}
We emphasize \emph{adaptive} dynamic regret because the current action $a_t$ not only determines the current reward but also affects the future cost functions through the transition kernel. Even this regret is weaker than $\sR_V$, exactly because of this adaptivity point. Formally,
\[ 
\E \big[ \sumt
\max_{a \in \cA} r( a, S_t)\big] \leq \max_{a_{1:T} \in \cA^T} \E \big[ \sumt r( a_t, S_t)\big] = V^*_T(s_1),
\]
which shows that $\sR_V$ considers a stronger competitor. The equality holds when the adaptivity collapses, and there is no planning effect in taking actions, as in contextual bandit setting, or in repeated games when the opponent's memory $m =0$, (see Section~\ref{application to games}). In that case, the dynamic regret coincides with $\sR_V$.

\subsection{Different Types of Opponents} \label{Different Types of Opponent}
Depending on $\psi$'s input, we focus on two types of opponents:

\parhead{General Opponent.} When we only assume an unknown limited memory of $m^*$ for the opponent, and $\psi$ might be depending on both the learner's and opponent's actions, i.e., $\psi: (\cA \times \cB)^{m^*} \rightarrow \Delta_\cB$. 

\parhead{Self-Oblivious Opponent.} which is the case that $\psi: \cA^{m^*} \rightarrow \Delta_\cB$  only depends on the learner's actions. This makes some specific structures on the underlying MDPs. In more detail, from proposition \ref{prop0} we know that we can define the state space only over the learner's actions. Furthermore, despite the general opponent, the transition probability will be known and deterministic. This is because the current state and action fully determine the next state $s_{t+1}$  (Refer to appendix \ref{deter. tran.} for more details). The opponent's unknown policy $\psi$ only contributes to the reward function. Based on this crucial assumption, \citep{arora2018policy} have defined generalized equilibrium.

In this case, the state space is $\cS = \cA^m$, and $s_t = (a_{t-1},\dots, a_{t-m})$. The reward $R_t = \bU(A_t,B_t)$ where $B_t \sim \psi(.|S_t)$. This means that the distribution of the reward function is also unknown. The transition dynamic, is a deterministic dynamic, $P : \cS \times \cA \rightarrow \cS$, where,
\begin{align} \label{deter. tran.}
P(s' | a,s) = \Prob[a'_{m:1} | a, a_{m:1}] =
\begin{cases}
    1 & \text{if}\  a'_m = a \ , \ a'_{m-i} = a_{m-i+1} \quad \forall i \in [m-1] \\
    0 & \text{otherwise }
\end{cases}
\end{align}
Note that the recent action, together with the current state, deterministically implies the next state.

One might also consider \emph{Learner-Oblivious Opponent} for which all contributes in $\psi$ is the opponent's actions, i.e., $\psi: \cB^{m^*} \rightarrow \Delta_\cB$.
This turns the setting into a contextual online learning problem, because the actions taken by the learner do not affect the distribution of the opponent's actions (the context). We do not consider this case because it is not natural to assume the opponent takes action only with respect to its own previous ones, with no care about how the learner responds.

\section{Missing Proofs}

\subsection{Proof of Proposition 5.5} \label{prop0 proof}

\begin{lemma} \label{state occupancy lemma}
     Against any $m$-th order policy $\psi$ and for any policy $\pi \in \Pi_\cA$, the policy $\pi'\in \bigcup_{i=0}^m \Pi_\cA^i$ in the order of at most $m$, with the description of 
\[
    \pi'(a|o_{m:1}) = \frac{\sumt \Prob_{\pi,\psi}(A_t = a, O_{t-1:t-m} = o_{m:1})}{\sumt \Prob_{\pi,\psi}(O_{t-1:t-m} = o_{m:1})} \quad \quad \forall a \in \cA \ \text{and}\  o_{m:1} \in \cO^m
\]
and the same initial distribution $\Prob_{\pi,\psi}(O_{0:1-m} = o_{0:1-m}) = \Prob_{\pi',\psi}(O_{0:1-m} = o_{0:1-m})$ over the actions, satisfies,
     \[
     \sum_{t = 1}^T \Prob_{\pi',\psi}(O_{t-1:t-m} = o_{m:1}) = \sum_{t = 1}^T \Prob_{\pi,\psi}(O_{t-1:t-m} = o_{m:1}) \quad \quad \forall o_{m:1} \in \cO^m
     \]
\end{lemma}
\begin{proof} we denote  $\lambda_{\pi,\psi}(o_{m:1}) = \sum_{t = 1}^T \Prob_{\pi,\psi}(O_{t-1:t-m} = o_{m:1})$ and $\lambda_{\pi,\psi}(a; h_{m:1}) = \sum_{t = 1}^T \Prob_{\pi,\psi}(A_t = a, O_{t-1:t-m} = o_{m:1})$.  Using this notation
\begin{align*}
    \pi'(a|o_{m:1}) = \frac{\lambda_{\pi,\psi}(a; o_{m:1})}{\lambda_{\pi,\psi}(o_{m:1})} \quad \quad \forall a \in \cA \ \text{and}\  o_{m:1} \in \cO^m
\end{align*}

Now we write
\begin{align}
    \lambda_{\pi,\psi}(o_{m:1})  & = \Prob_{\pi,\psi}(O_{0:1-m} = o_{m:1}) + \sum_{t = 1}^{T-1} \Prob_{\pi,\psi}(O_{t:t-m+1} = o_{m:1})\\
    & = \Prob_{\pi,\psi}(O_{0:1-m} = o_{m:1}) + \sum_{t = m+1}^{T-1} \sum_{o_{prev}\in\cO}  \Prob_{\pi,\psi}(A_t = a_m, O_{t-1:t-m} = o_{m-1:1},o_{prev}) \psi(b_m|o_{m-1:1},o_{prev})\\
    & = \Prob_{\pi,\psi}(O_{0:1-m} = o_{m:1}) + \sum_{o_{prev}\in\cO} \psi(b_m|o_{m-1:1},o_{prev}) \lambda_{\pi,\psi}(a_m; o_{m-1:1},o_{prev}) \\
    & = \Prob_{\pi,\psi}(O_{0:1-m} = o_{m:1}) + \sum_{o_{prev}\in\cO} \psi(b_m|o_{m-1:1},o_{prev}) \pi'(a_m|o_{m-1:1},o_{prev}) \lambda_{\pi,\psi}(o_{m-1:1},o_{prev})
\end{align}
where that last equality comes from the definition of $\pi'$. We can  start from $\lambda_{\pi',\psi}(h_{m:1})$, and reach the same equality, i.e.,
\begin{align}
    \lambda_{\pi',\psi}(h_{m:1}) = \Prob_{\pi',\psi}(O_{0:1-m} = o_{m:1})  + \sum_{o_{prev}\in\cO} \psi(b_m|o_{m-1:1},o_{prev}) \pi'(a_m|o_{m-1:1},o_{prev}) \lambda_{\pi',\psi}(o_{m-1:1},o_{prev})
\end{align}
By subtracting the two results and noting that the two policies have the same initial distributions $\Prob_{\pi,\psi}(O_{0:1-m} = o_{0:1-m}) = \Prob_{\pi',\psi}(O_{0:1-m} = o_{0:1-m})$, we will have,
\begin{align}
    \epsilon (o_{m:1}) = \sum_{o_{prev}\in\cO} \psi(b_m|o_{m-1:1},o_{prev}) \pi'(a_m|o_{m-1:1},o_{prev}) \epsilon(o_{m-1:1},o_{prev}) \quad \quad \forall o_{m:1} \in \cO^m
\end{align}
where $\epsilon (h_{m:1}) = \lambda_{\pi',\psi}(h_{m:1}) - \lambda_{\pi,\psi}(h_{m:1})$. This gives us a system of $|\cH|^m$ linear equations which has a unique solution of $\epsilon = 0$. This implies $ \lambda_{\pi',\psi}(h_{m:1}) = \lambda_{\pi,\psi}(h_{m:1})$.
\end{proof}
\begin{proposition}
     Against any $m$-th order policy $\psi$ and for any policy $\boldsymbol{\pi} \in \Pi_\cA$, there exists a policy $\pi'\in \bigcup_{i=0}^m \Pi_\cA^i$ in the order of at most $m$, where $\pi$ and $\pi'$ have the same value, i.e., $V^{\psi,\boldsymbol{\pi}}_T(s_1) = V^{\psi,\pi'}_T(s_1)$ for all $s_1 \in \cS_m$ and $T \in \N$.
\end{proposition}
\begin{proof}  Let $\boldsymbol{\pi}= (\pi_1,\dots,\pi_T)$ be an arbitrary policy, where $\pi_t: \cO^{t-1} \rightarrow \Delta_\cA$.
We show that the policy introduced in lemma \ref{state occupancy lemma} is the solution. We start by writing the value of $\boldsymbol{\pi}$,
\begin{align*}
    V^{\psi,\boldsymbol{\pi}}_T(s_1)& = \sumt \sum_{a,b \in \cA\times \cB} \bU(a,b) \Prob_{\boldsymbol{\pi},\psi}(A_t =a,B_t =b) \\
    & =  \sum_{a,b \in \cA\times \cB} \bU(a,b)\sumt \sum_{o_{t-1:1} \in \cO^{t-1}} \Prob_{\boldsymbol{\pi},\psi}(A_t =a,B_t =b|O_{t-1:1} = o_{t-1:1}) \Prob_{\boldsymbol{\pi},\psi}(O_{t-1:1} = o_{t-1:1})\\
    & =  \sum_{a,b \in \cA\times \cB} \bU(a,b) \sumt \sum_{o_{t-1:1} \in \cO^{t-1}} \psi(b|o_{t-1:1}) \pi_t(a|o_{t-1:1}) \Prob_{\pi,\psi}(O_{t-1:1} = o_{t-1:1}) \\
\end{align*}
As $\psi$ is $m$-th order,
\begin{align*}
     V^{\psi,\boldsymbol{\pi}}_T(s_1) & =  \sum_{a,b \in \cA\times \cB} \bU(a,b) \sumt \sum_{o_{t-1:1} \in \cO^{t-1}} \psi(b|o_{t-1:t-m}) \pi_t(a|o_{t-1:1}) \Prob_{\pi,\psi}(O_{t-1:1} = o_{t-1:1}) \\
    & =  \sum_{a,b \in \cA\times \cB} \bU(a,b) \sum_{o_{m:1} \in \cO^m } \psi(b|o_{m:1}) \sumt  \sum_{o_{t-m-1:1} \in \cO^{t-m-1}} \pi_t(a|o_{m:1},o_{t-m-1:1}) \Prob_{\pi,\psi}(O_{t-1:1} = o_{m:1},o_{t-m-1:1}) \\
    & =  \sum_{a,b \in \cA\times \cB} \bU(a,b) \sum_{o_{m:1} \in \cO^m } \psi(b|o_{m:1}) \sumt \Prob_{\pi,\psi}(A_t = a, O_{t-1:t_m} = o_{m:1}) \\
    & =  \sum_{a,b \in \cA\times \cB} \bU(a,b) \sum_{o_{m:1} \in \cO^m } \psi(b|o_{m:1}) \pi'(a|o_{m:1}) \sumt \Prob_{\pi,\psi}(O_{t-1:t-m} = o_{m:1}) =V^{\psi,\pi'}_T(s_1)\\
\end{align*}

Note that the last equation comes from the previous lemma which states
\[
\sumt \Prob_{\pi',\psi}(O_{t-1:t-m} = o_{m:1}) = \sumt \Prob_{\pi,\psi}(O_{t-1:t-m} = o_{m:1}) \quad \quad \forall o_{m:1} \in \cO^m,
\]
and we assume $\boldsymbol{\pi}$ and $\pi'$ are initialized with a state coming from the initial distribution $\mu$ of MDP. 
\end{proof}
\begin{proposition}
     Against any $m^*$-th order policy $\psi$ in the game setting the following property holds,
     \begin{align}
         g^*_{m^*} = g^*_{m^* +1} = \dots = g^*_{M-1}
     \end{align}
\end{proposition}
\begin{proof}
We show $g^*_{M-1} = g^*_i$ for all $ i \in [m^* : M-2]$. Suppose, $\pi^*_{M-1}$ is the optimal policy achieving the optimal gain of $g^*_{M-1}$. We know that $g^*_i \leq g^*_{M-1}$ , since it is the maximum value over a smaller set. Thus we only need to show $g^*_i \geq  g^*_{M-1}$. Since $\psi$ is an $m^*$-th order policy, from proposition \ref{prop0} we know that for all $ i \in [m^* : M-2]$ there exists a policy $\pi_i$ in the order of at most $i$corresponded to $g^*_{M-1}$, such that for all $T \in \N$, $V^{\psi,\pi^*_{M-1}}_T(s_1) = V^{\psi,\pi_i}_T(s_1)$. Therefore by taking the limitations we also have,
\begin{align*}
   g^*_{M-1} =  \lim_{T \rightarrow \infty} \frac{1}{T} V^{\psi,\pi^*_{M-1}}_T(s_1) = \lim_{T \rightarrow \infty} \frac{1}{T}  V^{\psi,\pi_i}_T(s_1) = g_i \leq g^*_i.
\end{align*}
This concludes the proof.

As we are in the game setting, we know that $\pi^*_{m^*}$ is also a valid policy in the model classes of higher orders $i \in [m^*,M-1]$, since the action space remains the same for all model classes and $\pi^*_{m^*}$ can only consider $m^*$ many previous interactions between the learner and opponent, even if more interactions are provided in the larger model classes. This means that $\pi^*_{m^*}$ is an optimal policy for all MDPs that $\psi$ induces in all $\cC_{m^*}, \dots,\cC_{M-1} $. Therefore the rewards it gathers are the same. Thus the optimal bias of an state $h^*_i(s) : =  \bE_s^{\pi^*_{m^*}} [\lim_{T\rightarrow \infty}\sumt(R_t - g(S_t))]$, only depends on the first $m^*$ many interactions, and therefore, $\spn(h^*_i)$ is invariant for $i \in [m^* , M-1]$. This allows us to denote it by $\spn(h^*)$.
\end{proof}

\subsection{Proof of Upper Bound on $\spn(h^*)$} \label{proof of upperbound on span}
In this section, we study how different choices of $\psi$ with the same memory of $m$ affect the span of optimal bias and give an upper bound on $\spn(h^*)$ based on the Kemeny's index of induced by $\psi$. This shows how we can have the upper bound $c_{h^*}$ on $\spn(h^*)$. Span of $h^*$ is a parameter that appears in many problem-dependent regret bounds in the average reward RL literature \cite{boone2024achieving, zhang2019regret, fruit2018SCAL}. 
From the Poisson equation, we have $(I - P^\pi)h^\pi = (I - P^{\pi}_\infty)r^\pi$, where $P^\pi_\infty = lim_{N\rightarrow \infty} \frac 1 N \sum_{n =0}^\N (P^\pi)^n$.  In the aperiodic Markov chains, the limit is well-defined. Note that $g^\pi = P^\pi_\infty r^\pi$. The fact that the rows of $P^\pi_\infty$ are the stationary distribution of the Markov chain generated by $P^\pi$ , denoted by $\mu^\pi_\infty$, helps in giving the intuition and why $g^\pi \in \R \boldsymbol{1}$ (in ergodic MDPs). The optimal bias of a Markov chain can be computed as $h^* = (I-P^*)^\# r^* = (I - P^* +P^*_\infty)^{-1}(I-P^*_\infty) r^*$, where $A^\#$ is the group inverse of a singular matrix $A$, and superscript $*$ is to note the use of optimal policy. $Z^\pi:=(I - P^\pi +P^\pi_\infty)^{-1}$ is called the fundamental matrix, also $H^\pi: = (I-P^\pi)^\#$ is  called the deviation matrix.

To upper bound the span of optimal bias, we use the ergodicity coefficient \cite{seneta1984explicit} \cite{seneta1993sensitivity},

\begin{definition}
\cite{seneta1993sensitivity} For any $n\times n$ matrix $A = \{ a_{ij}\}$ with all its rows sum up to $a$ (i.e. $A\boldsymbol{1} = a \boldsymbol{1}$ ), the ergodicity coefficient of $A$ is,
\[
\tau_1(A) =  \sup_{ \{ \delta: \Vert\delta\Vert_1 = 1, 
\delta \boldsymbol{1} = 0 \}} \Vert \delta A\Vert_1 = \max_{i,j} \frac{1}{2} \sum_{k = 1}^n |A_{is} - A_{js}|
\]
\end{definition}
Ergodicity coefficient can be interpreted as a norm for the matrix $A$, where it captures the maximum total variation between a pair of the rows of $A$. Note $Z^\pi$ and $H^\pi$ have the rows that sum up to one and zero respectively. The following lemma leads us to the upper bound 

\begin{lemma}
\label{upperbound on tau} \cite{seneta1993sensitivity} For an ergodic transition matrix $P$ and its corresponding fundamental matrix $Z$ and deviation matrix $H$,
\[
\frac{1}{\min_{2 \leq k \leq n} |1 - \lambda_k|}\leq \tau_1 (H) = \tau(Z) \leq \sum_{k = 2}^n \frac{1}{|1 - \lambda_k|}
\]
where $\lambda_i$s are the eigenvalues of $P$, enumerated such as $1 = \lambda_1 > |\lambda_2| \geq \dots \geq |\lambda_n|$.
\end{lemma}
Note that the set of eigenvalues of $H$ is $\{ 0, (1 - \lambda_2)^{-1}, \dots, (1 - \lambda_n)^{-1} \}$, so the upper bound in \ref{upperbound on tau} is $tr(H)$ which is also called the Kemeny's index of the associated Markov chain.

\begin{proposition}
    For any policy $\pi$ in an ergodic or unichain MDP we have,
    \[
    \spn(h^\pi) \leq 2 \spn(r^\pi)\kappa^\pi,
    \]
where $h^\pi$ and $r^\pi$ are respectively the bias and reward vectors induced by $\pi$ and $\kappa^\pi$ is the Kemeny's constant of the Markov chain with transition kernel $P^\pi$. The same inequality holds in weakly communicating MDPs for the optimal policy $\pi^*$,
\begin{align}
        \spn(h^*) \leq 2 \spn(r^*)\kappa^*.
\end{align}
\end{proposition}
\begin{proof}
     Define $\bar r^\pi = r^\pi - g^\pi$, as $P^\pi_\infty \bar r^\pi = \boldsymbol{0}$, we have $\spn( r ^\pi) = \spn(\bar r ^\pi) \geq \max_{i} \bar r ^\pi_i $. Note that in ergodic and/or unichain MDPs the gain of every policy is a constant vector in $\R\boldsymbol{1}$. Also consider $i = \arg \max_k |Z_k^\pi \bar r^\pi|$ and $j = \arg \max_k |Z_k^\pi \bar r^\pi|$ where $Z^\pi_k$ is the $k$-th row of $Z^\pi$. Now we can write,
\begin{align*}
    \spn(h^\pi) & = \spn(Z^\pi \bar r \pi) = |\sum_{k  =1}^n Z^\pi_{ik}\bar r ^\pi_k| + |\sum_{k  =1}^n Z^\pi_{jk}\bar r ^\pi_k| \\
    & \leq  |\sum_{k  =1}^n \bar r ^\pi_k( Z^\pi_{ik} - Z^\pi_{jk})| \\
    & \leq \spn(\bar r ^\pi)\sum_{k  =1}^n | Z^\pi_{ik} - Z^\pi_{jk}| \\
    & \leq 2 \spn(r^\pi)\tau_1(Z^\pi) \leq 2 \spn(r^\pi) \kappa^\pi
\end{align*}
The same can be done for the optimal policy in a weakly communicating MDP, where the optimal gain is constant.
\end{proof}
In the repeated game setting, since the range of the utility function is assumed to be $[0,1]$, $\spn(r^\pi) \leq 1$ for all $\pi$. 
\subsubsection{Span in Self-Oblivious Case}
In the self-oblivious case, by leveraging on the structure of the MDP, we can have the following upper bound on the $\spn(h^*)$,
\begin{proposition}
    In a self-oblivious model $\cM$ with the order of $m$ (described in appendix \ref{deter. tran.}), for every policy $\pi$ the upper bound $ D(\cM) \leq m $ on the diameter holds. This implies $\spn(h^*) \leq  m \spn(r^*)$.
\end{proposition}
\begin{proof}
     As the transition kernel is deterministic and the state space $\cS = \cA^m$, the learner has full control of the stream of states he wants to visit. Thus the learner can reach any state $s'$ from any start state $s$ after at most $m$ iteration by taking the actions of $s'$. This shows $D(\cM)\leq m$. It is also a well-known result that $\spn(h^*) \leq \spn(r^*)D(\cM)$. One can consider the above travel from $s$ to $s'$ as a prefix of the optimal policy starting from $s'$ which shows that $h^*(s)$ can be at most $\spn(r^*)D(\cM)$ less than $h^*(s')$.
\end{proof}


\subsection{Gap between $V_T^*$ and $T g^*$}
\begin{lemma}
    \label{gvtgap}For a weakly communicating MDP, we have,
    \[
    \Vert V_T^* - T g^*\Vert_\infty \leq 2 \Vert h^* \Vert_\infty
    \]
\end{lemma}
\begin{proof}
    We know $V_T^* = L^T \boldsymbol{0}$, where $L$ is the optimal bellman operator and $\boldsymbol{0}$ is the vector full of zeros. Also it we know $g^* + h^* = L h^*$ and, because the MDP is weakly communicating, the $g^*$ is a constant vector, which means $L^2 h^* = L(g^* + h^*) = g^* + L(h^*) = 2 g^* + h^*$. By Induction you will get $L^T h^* = Tg^* + h^*$  \cite{puterman2014markov}. Therefore we can write,
\begin{align*}
    \Vert V_T^* - T g^*\Vert_\infty  & = \Vert L^T\boldsymbol{0} - L^T h^* + h^*\Vert_\infty \\
    & \leq \Vert L^T\boldsymbol{0} - L^T h^* \Vert_\infty + \Vert h^*\Vert_\infty\\
    &\leq 2 \Vert h^* \Vert_\infty
\end{align*}
where the last inequality holds because $L$ is a non-expansive operator.
\end{proof}

\subsection{Confidence Regions $\boldsymbol{\cM}_t$ by Estimations on $\psi$} \label{properM_tbypsi}

The algorithm $\pmdt$ \cite{boone2024achieving}, gets a system of confidence regions $\boldsymbol{\cM}_t$ as the input, and its guarantee on the regret holds if this system satisfies some assumptions. In this section we propose those assumptions and show how we can satisfy them in the game setting. Fix a model class,
\begin{assumption}
  \label{assump1}  with probability $1-\delta$, we have $M \in \cap _{k  = 1}^{\K(T)} \boldsymbol{\cM}_{N_{\K(T)}}$  
\end{assumption}
This is a common achievable assumption that is usually captured by empirical estimation of transition kernel and reward function. For further explanation refer to appendix A.2 of \citep{boone2024achieving}. In our setting, we can empirically estimate the distributions of $\psi$ and by transferring from $\hat \psi$ to transition kernel and reward function, get some estimation that is satisfying above assumption (section \ref{RL Setting}). Note that the utility function for the learner is known.
\begin{assumption}
   \label{assump2} There exists a constant $C >0$ such that for all $(s,a)$, for all $t \leq T$, we have,
    \begin{align*}
        \cR_t(s,a) \subseteq \{ \tilde r(s,a): N_t(s,a) \Vert \hat r (s,a) - \tilde r(s,a)\Vert_1^2  \leq C \log (2SA(1 + N_t(s,a)) /\delta)\} 
    \end{align*}
\end{assumption}
This assumption also emphasizes that the reward function satisfies a concentration on the differnce between the optimistic bonused estimation of reward and the empirical estimation of it. As we have a bounded utility function and discrete distribution over the opponent's action $\psi$ we satisfy this assumption too. Refer to appendix A.2.3 \cite{boone2024achieving} for explanation of this assumption.
\begin{assumption}
   \label{assump3} For $t\geq 0$, $\boldsymbol{\cM}_t$ is a $(s,a)$-rectangular convex region and $L_t^n(u)$ converges a fixed point.
\end{assumption}
After estimating $\psi$, based on that the reward function and the transition kernel will be estimated and the confidence region will be the rectangular product of their confidence regions. Also the bellman operator upon which $\mathsf{PMEVI}$ adds mitigagion and projection, converge in typical cases including us \citep{boone2024achieving}. In \citep{boone2024achieving}, there is another assumption to make sure the set of candidate optimal bias functions includes the true optimal bias, which is not used in this paper.







\subsection{Concentration Inequality Based on \cite {boone2024achieving} Result}

In this section, we provide a concentration inequality on the optimal gain, so that it can be used in the model selection algorithm. So suppose we are considering only one model class with optimal gain and optimal bias denoted by $g^*$ and $h^*$ respectively. It is based on the main result of \cite{boone2024achieving}, which is,

\begin{proposition}
    (\cite{boone2024achieving}, Theorem 5) Let $c>0$. Assume that $\pmdt$ runs with a confidence region system, $t \rightarrow \mathcal{M_t}$ that guarantees assumptions \ref{assump1}, \ref{assump2}, \ref{assump3}. If $T \geq c^5$, then for every weakly communicating model model with $\spn(h^*) \leq c$, $\pmdt$ achieves regret,
\begin{align}
    \tilde O \big ( \sqrt{cSAT \log(SAT/\delta)} \big) + \tilde O \big( c S^{\frac{5}{2}} A^{\frac{3}{2}} T^{\frac{9}{20}} \log^2(SAT/\delta)\big)
\end{align}
with probability $1-26\delta$. And also in expectation if $\delta < \sqrt{1 / T}$.
\end{proposition}
Also from the proof of the proposition and noting to \cite{boone2024achieving} (lemma 6), with probability $1-26\delta$ we have,
\begin{align}
\label{g&gstar}
    T g^* - \sumt R_t \leq  \sumk \sum_{t_k}^{t_{k+1}-1} \tg_k - R_t \in \tilde O \big ( \sqrt{cSAT \log(SAT/\delta)} \big) + \tilde O \big( c S^{\frac{5}{2}} A^{\frac{3}{2}} T^{\frac{9}{20}} \log^2(SAT/\delta)\big)
\end{align}
where $\K(T)$ is the number of epochs that $\pmdt$ takes, $\tg_k$ is the projected optimistic gain calculated in the epoch $k$ by the algorithm (denoted by $\mathfrak{g}_k$ in \cite{boone2024achieving}).
From lemma \ref{gvtgap} we know,
\[
- 2 \spn(h^*) \leq T g^* - \sumt R_t 
\]
which gives us,
\[
\sumt R_t - 2 \spn(h^*) \leq T  g^* \leq \sumk n_k \tg_k \leq  \sumt R_t +\tilde O \big ( \sqrt{cSAT \log(SAT/\delta)} \big) + \tilde O \big( c S^{\frac{5}{2}} A^{\frac{3}{2}} T^{\frac{9}{20}} \log^2(SAT/\delta)\big)
\]
and thus,
\[
\frac{\sumk n_k \tg_k}{T} - g^* \leq \tilde O \big ( \sqrt{\frac{cSA \log(SAT/\delta)}{T}} \big) + \tilde O \big( c S^{\frac{5}{2}} A^{\frac{3}{2}} T^{-\frac{11}{20}} \log^2(SAT/\delta)\big) + \frac{2\spn(h^*)}{T}.
\]
The leading term is the first term which means,

\begin{lemma}
\label{regconcentration}
By running $\pmdt$  for a big enough $T$, there exists a universal constant $C_B$ such that we can write,
\begin{align}
    \frac{\sumk n_k \tg_k}{T} - g^* \leq   C_B \sqrt{\frac{cSA \log(SAT/\delta)}{T}}
\end{align}
with probability $1-26\delta$. Or equivalently,
\begin{align}
    \frac{\sumk n_k \tg_k}{T} - g^* \geq  \epsilon
\end{align}
with probability $26 SAT \exp{(-\frac{\epsilon^2 T}{C_B^2 c SA})}$.
\end{lemma}
where $n_k$ is the number of iterations passed in epoch $k$. We also define a cumulative counter $N_k = \sum_{l = 1}^k n_l$. What inequality should $T$ and $C_B$ satisfy?

\subsection{Upper bound on $\K(T)$ Under Doubling Trick}
\begin{lemma}
\label{upperboundoonKT}
\cite{Jaksch2008near} The number of epochs up to time $T \geq SA$ under doubling trick, is upper bounded by,
\[
\K(T) \leq SA\log_2(\frac{8T}{SA})
\]
\end{lemma}

\subsection{$\cG$ is Highly Probable}
\begin{lemma}
\label{Gishighprob}   
 $\Prob(\cG) \geq 1 - 26(M-m^*)(\K(T)+1)\delta$, where
    \[
    \cG = \{ \forall i\in [m^*:M-1] ,  k \in [1:\K(T)+1] : \Reg_i(\pmdt,N_{i,k}) \leq C_i\sqrt{ N_{i,k} \log(\frac{N_{i,k}}{\delta})} \}
    \]
\end{lemma}
\begin{proof}
Define the following events for all $i\in [m^*:M-1]$ and $ k \in [1:\K(T)+1] $,
\[
\cG_{i,k} = \{ \Reg_i(\pmdt, N_{i,k}) \leq \ C_i\sqrt{ N_{i,k} \log(\frac{N_{i,k}}{\delta})} \}
\]
As $N_{i,k} \geq c_{h^*}^5$ for all $1\leq k \leq \K(T)$, and $i$ lies in the well-specified model classes, from \ref{baseguarantee}, we know
\[
\Prob(\cG_{i,k}) \geq 1- 26\delta
\]
from union bound we,
\[
\Prob(\cG^c) = \Prob(\cup_{i,k} \cG_{i,k}^c) \leq \sum_{i,k} \Prob(\cG_{i,k}^c) \leq 26 (M-m^*) (\K(T)+1) \delta,
\]
which concludes the proof.
\end{proof}

\subsection{Regret is Multiplicatively Balanced}
\begin{lemma}
\label{Regretisbalanced}
By running algorithm $\bear$ with base algorithms of $\pmdt$, for all $0< \delta < 1$, $k \in \N$ and any pair of $i ,j \in \cI_k $ where $i \neq j$, and when $\bb_i(N_{i,k},\delta) = C_i\sqrt{N_{i,k} \log(\frac{N_{i,k}}{\delta})}$, the followings hold,
\begin{enumerate}
    \item 
\[
\bb_i(N_{i,k}, \delta) \leq \bb_j(N_{j,k},\delta) + \alpha \bb_i(N_{i,k-1},\delta)+ \beta
\]
\item
\begin{align*}
    \label{NoverN}
\frac{N_{i,k}}{N_{j,k}} \leq \big (\frac{C_j}{(1-\alpha)C_i} + \frac{\beta}{(1-\alpha)C_i\sqrt{N_{j,k} \log(N_{j,k}/\delta)}}\big)^2 \log(\frac{N_{j,k}}{\delta})
\end{align*}

\end{enumerate}
where $\alpha = \frac{\log(c_{h^*}^5 \lor 9)+1}{2\log(c_{h^*}^5 \lor 9)}$  and  $\beta = \bb_{M-1}(c_{h^*}^5,\delta) -\bb_0(c_{h^*}^5,\delta) $.
\end{lemma}
\begin{proof}
Based on \ref{baseguarantee} we have defined $\bb_i(N_{i,k},\delta) = C_i\sqrt{N_{i,k} \log(\frac{N_{i,k}}{\delta})}$, therefore we have,
\[
\bb_i(N_{i,k},\delta) - \bb_i(N_{i,k-1},\delta) \leq \bb_i(2N_{i,k-1},\delta) - \bb_i(N_{i,k-1},\delta) \leq \frac{1}{2} C_i\big (\sqrt{ N_{i,k-1} \log(\frac{N_{i,k-1}}{\delta})} + \sqrt{\frac{N_{i,k-1} \delta^2}{\log(N_{i,k-1}/\delta)}} \big )
\]
where the second inequality holds because $\bb_i(t,\delta)$ is increasing and concave in $t$, and the first one is due to the doubling trick ($N_{i,k} \leq 2 N_{i,k-1}$). Then as $\delta \leq 1$ and $N_{i,k} \geq c_{h^*}^5 \lor 9$ we have,
\begin{align} \label{alphaBi}
\bb_i(N_{i,k},\delta) - \bb_i(N_{i,k-1},\delta) &\leq \frac{1}{2} C_i\big (\sqrt{ N_{i,k-1} \log(\frac{N_{i,k-1}}{\delta})} + \sqrt{\frac{N_{i,k-1} \delta^2}{\log(N_{i,k-1}/\delta)}} \big ) \\
& \leq (\frac{1}{2} + \frac{1}{2\log(c_{h^*}^5 \lor 9)}) \bb_i(N_{i,k-1},\delta) \\
& = \alpha  \bb_i(N_{i,k-1},\delta)
\end{align}
Before showing the main statement, note $\bb_i(t,\delta)$ is strictly increasing in $i \in \N\cup \{ 0\}$ and $t \in \N$. Also for reducing redundancy we temporary drop $\delta$ and just write $\bb_i(t)$.
We show the lemma statement by induction:
For $k=1$, note that $N_{m,0} =c_{h^*}^5$ for all $m \in [0:M-1]$. Suppose there exists a pair of $i,j \in \cI_1$ such that $\bb_i(N_{i,1}) > \bb_j(N_{j,1}) + \alpha \bb_i(c_{h^*}^5)+ \beta$, and it will give contradiction. This means that the algorithm has picked model class $i$ in epoch k, i.e. $i_k = i$, since if this would not be the case, $N_{i,1} =N_{i,0} = c_{h^*}^5$ and the inequality of the lemma should have been satisfied (Note the definition of $\beta$). Therefore, $\bb_i(N_{i,1}) > \bb_j(N_{j,1}) + \alpha \bb_i(c_{h^*}^5)+ \beta$ implies $i_k = i$. Now we can write,
\begin{align*}
\bb_i(c_{h^*}^5) &\geq \bb_i(N_{i,1}) - \alpha \bb_i(c_{h^*}^5) \\
& > \bb_j(N_{j,1}) + \beta \\
& = \bb_j(c_{h^*}^5) + \beta \\
& \geq \bb_j(c_{h^*}^5)
\end{align*}
which contradicts with $i = \arg \min_{m\in \cI_1} \bb_m(c_{h^*}^5)$. The first inequality is from \ref{alphaBi}, the strict inequality is from the contradiction assumption, and the equality is for $ j \neq i = i_k$.

Induction step: Again for the sake of contradiction, suppose $k$ is the first epoch in which the lemma's inequality violates, i.e.
\[
\bb_i(N_{i,k}) > \bb_j(N_{j,k}) + \alpha \bb_i(N_{i,k-1})+ \beta.
\]
For a similar reason as before, this implies that $i_k = i$. And by similar chain of inequalities we get to a contradiction with $i = \arg \min_{m\in \cI_k} \bb_m(N_{m,k-1})$ ,
\begin{align*}
\bb_i(N_{i,k-1}) &\geq \bb_i(N_{i,k}) - \alpha \bb_i(N_{i,k-1}) \\
& > \bb_j(N_{j,k}) + \beta \\
& = \bb_j(N_{j,k-1}) + \beta \\
& \geq \bb_j(N_{j,k-1}).
\end{align*}
This concludes the proof for the first part. 
Now we move on to the second statement. From the first part, we have,
\[
\bb_i(N_{i,k}) \leq \bb_j(N_{j,k})+\alpha \bb_i(N_{i,k-1})+ \beta,
\]
and by replacing $\bb_i(N_{i,k-1})$ with $\bb_i(N_{i,k})$, the following holds,
\begin{align}
\label{temp6}
\bb_i(N_{i,k}) \leq \frac{1}{1-\alpha}\bb_j(N_{j,k})+ \frac{\beta}{1-\alpha}.
\end{align}
Note that $\alpha \leq \frac{3}{4}$. Replacing $\bb_i(N_{i,k},\delta) = C_i\sqrt{N_{i,k} \log(\frac{N_{i,k}}{\delta})}$ gives us,
\[
C_i \sqrt{N_{i,k}\log(N_{i,k} /\delta)} \leq \frac{1}{1-\alpha} C_j \sqrt{N_{j,k}\log(N_{j,k}/ \delta)} + \frac{\beta}{1-\alpha}
\]
Now by rearranging the terms we have,
\[
\frac{\sqrt{N_{i,k}\log(N_{i,k} / \delta)}}{\sqrt{N_{j,k}\log(N_{j,k}/\delta)}} \leq \frac{C_j}{(1-\alpha)C_i} + \frac{\beta}{(1-\alpha)C_i \sqrt{N_{j,k}\log(N_{j,k}/\delta)}}
\]
By squaring both sides and noting that $\log(N_{i,k}/\delta) \geq 1$, we reach the second statement.
\end{proof}

\subsection{$\Reg_i$ Upper Bound Based on $\frac{N_{i,k-1}}{N_{m^*,k-1}}$ for Misspecified $i < m^*$}
\begin{lemma}
    \label{reg_iboundwithN}
    For any active model class $i \in \cI_k$  such that $i<m^*$, under the event $\cG$, the regret of model class $i$ is bounded as follows,
    \begin{align}
      \Reg_i(N_{i,k-1})
\leq  (\frac{N_{i,k-1}}{N_{m^*,k-1}} + \frac{1}{1-\alpha})\bb_{m^*}(N_{m^*,k-1}) +\frac{2 N_{i,k-1}}{N_{m^*,k-1}} c_{h^*} + \frac{\beta}{1-\alpha}.
    \end{align}
\end{lemma}
\begin{proof}
When $i \in \cI_k$, it should have passed the miss specification test. Thus as $m^* > i$, the following holds
\[
  \frac{\bb_i(N_{i,k-1},\delta) + \sum_{t \in \cN_{i,k-1}} r_t}{N_{i,k-1}} \geq \frac{\sum_{t \in \cN_{m^*,k-1}} r_t - 2 c_{h^*}}{N_{m^*,k-1}}.
\]
We temporarily omit $\delta$ as this input does not change. Subtracting $g^\star$ from both sides gives us,
\[
\frac{\bb_i(N_{i,k-1})}{N_{i,k-1}} - \frac{\Reg_i(N_{i,k-1})}{N_{i,k-1}} \geq -(\frac{\Reg_{m^*}(N_{m^*,k-1})+ 2 c_{h^*}}{N_{m^*,k-1}})
\]
by rearranging the terms we have,
\begin{align}
\label{temp2}
\Reg_i(N_{i,k-1}) &\leq \bb_i(N_{i,k-1}) + \frac{N_{i,k-1}}{N_{m^*,k-1}}(\Reg_{m^*}(N_{m^*,k-1}) + 2 c_{h^*}) \\
& \leq \bb_i(N_{i,k-1}) + \frac{N_{i,k-1}}{N_{m^*,k-1}}(\bb_{m^*}(N_{m^*,k-1}) + 2 c_{h^*}) 
\end{align}
where the second inequality is implied as $m^*$ is well specified \ref{prop0}.
From lemma \ref{Regretisbalanced}, for $i\neq j \in \cI_{k-1}$, we have,
\begin{align}
\bb_i(N_{i,k-1}) \leq \frac{1}{1-\alpha}\bb_j(N_{j,k-1})+ \frac{\beta}{1-\alpha}.
\end{align}
Note that $\alpha \leq \frac{3}{4}$. Obviously, when $i$ is in $\cI_k$ it had been also in $\cI_{k-1}$. Also under $\cG$ we know all the wellspecified model classes, including $m^*$ remain active. So we can apply the above bound with $j = m^*$ on \ref{temp2} and get the lemmas claim,
\[
\Reg_i(N_{i,k-1})
\leq  (\frac{N_{i,k-1}}{N_{m^*,k-1}} + \frac{1}{1-\alpha})\bb_{m^*}(N_{m^*,k-1}) +\frac{2 N_{i,k-1}}{N_{m^*,k-1}} c_{h^*} + \frac{\beta}{1-\alpha}.
\]
\end{proof}

\subsection{Final Bound $\Reg_i$ for Misspecified $i <m^*$}
\begin{lemma}
\label{finalReg_ibound}
For any active model class $i \in \cI_{\K(T)+1}$  such that $i<m^*$, under the event $\cG$, the regret of model class $i$ is bounded as follows,
\begin{align}
    \Reg_i(N_{i,\K(T)}) \leq \frac{C_{m^*}^3 \sqrt{N_{m^*,\K(T)}} \log^2(N_{m^*,\K(T)} /\delta)}{(1-\alpha)^2 C_i^2 } + \frac{1}{1-\alpha}\bb_{m^*}(N_{m^*,\K(T)})+ O(\log^\frac{3}{2}(T/\delta))
\end{align}
\end{lemma}
\begin{proof} We continue on the result of \ref{reg_iboundwithN} using \ref{NoverN}. Consider
\begin{align}
\label{temp4}
      \Reg_i(N_{i,\K(T)})
\leq  (\frac{N_{i,\K(T)}}{N_{m^*,\K(T)}} + \frac{1}{1-\alpha})\bb_{m^*}(N_{m^*,\K(T)}) +\frac{2N_{i,\K(T)}}{N_{m^*,\K(T)}} c_{h^*} + \frac{\beta}{1-\alpha},
    \end{align}

and we bound it term by term. 
Denote $\gamma_i = \big (\frac{C_{m^*}}{(1-\alpha)C_i} + \frac{\beta}{(1-\alpha)C_i\sqrt{N_{m^*,\K(T)} \log(N_{m^*,\K(T)}/\delta)}}\big) =\frac{\bb_{m^*}(N_{m^*,\K(T)}) +\beta}{(1-\alpha)C_i\sqrt{N_{m^*,\K(T)} \log(N_{m^*,\K(T)}/\delta)}} $. We write,
\begin{align}
    \frac{N_{i,\K(T)}}{N_{m^*,\K(T)}} \bb_{m^*}(N_{m^*,\K(T)}) & \leq \gamma_i^2 \log(N_{m^*,\K(T)} /\delta)\bb_{m^*}(N_{m^*,\K(T)}) \\
    & =  \frac{(\bb_{m^*}(N_{m^*,\K(T)}) +\beta)^2 \bb_{m^*}(N_{m^*,\K(T)})}{(1-\alpha)^2C_i^2N_{m^*,\K(T)}} \\
    & = \frac{C_{m^*}(\bb_{m^*}(N_{m^*,\K(T)}) +\beta)^2 \log(N_{m^*,\K(T)} /\delta)}{(1-\alpha)^2 C_i^2 \sqrt{N_{m^*,\K(T)}}},
\end{align}
where the first inequality is due to \ref{NoverN}, and the further equalities are from the definition of $\bb_{m^*}$. Now we inspect the last term,
\begin{align*}
    \frac{C_{m^*}(\bb_{m^*}(N_{m^*,\K(T)}) +\beta)^2 \log(N_{m^*,\K(T)} /\delta)}{(1-\alpha)^2 C_i^2 \sqrt{N_{m^*,\K(T)}}} & = \frac{C_{m^*}\bb_{m^*}(N_{m^*,\K(T)})^2 \log(N_{m^*,\K(T)} /\delta)}{(1-\alpha)^2 C_i^2 \sqrt{N_{m^*,\K(T)}}}  \quad (:= \chi_1)\\
    & + \frac{2 C_{m^*}\beta \bb_{m^*}(N_{m^*,\K(T)})\log(N_{m^*,\K(T)} /\delta)}{(1-\alpha)^2 C_i^2 \sqrt{N_{m^*,\K(T)}}} \quad (:= \chi_2)\\
    & + \frac{C_{m^*}\beta^2 \log(N_{m^*,\K(T)} /\delta)}{(1-\alpha)^2 C_i^2 \sqrt{N_{m^*,\K(T)}}} \quad (:= \chi_3)
\end{align*}
For the first term we have,
\[
\chi_1 = \frac{C_{m^*}\bb_{m^*}(N_{m^*,\K(T)})^2 \log(N_{m^*,\K(T)} /\delta)}{(1-\alpha)^2 C_i^2 \sqrt{N_{m^*,\K(T)}}} = \frac{C_{m^*}^3 \sqrt{N_{m^*,\K(T)}} \log^2(N_{m^*,\K(T)} /\delta)}{(1-\alpha)^2 C_i^2 }.
\]
For the second term, note that $N_{i,\K(T)} \leq T$ for all $i$, so
\begin{align*}
    \chi_2 = \frac{2 C_{m^*}\beta \bb_{m^*}(N_{m^*,\K(T)})\log(N_{m^*,\K(T)} /\delta)}{(1-\alpha)^2 C_i^2 \sqrt{N_{m^*,\K(T)}}} = \frac{2 C_{m^*}^2\beta \log^{3/2}(N_{m^*,\K(T)} /\delta)}{(1-\alpha)^2 C_i^2 } \in O(\log^\frac{3}{2}(T/\delta))
\end{align*}
and for the third term,
\[
\chi_3 = \frac{C_{m^*}\beta^2 \log(N_{m^*,\K(T)} /\delta)}{(1-\alpha)^2 C_i^2 \sqrt{N_{m^*,\K(T)}}} \in o(1)
\]
Therefore, $\chi_1$ is the dominant term and we can wrap up this part by saying,
\begin{align}
\label{term1}
\frac{N_{i,\K(T)}}{N_{m^*,\K(T)}} \bb_{m^*}(N_{m^*,\K(T)})   \leq \frac{C_{m^*}^3 \sqrt{N_{m^*,\K(T)}} \log^2(N_{m^*,\K(T)} /\delta)}{(1-\alpha)^2 C_i^2 } + O(\log^\frac{3}{2}(T/\delta)).
\end{align}
Now we move on to the other terms of \ref{temp4}. We keep the second term as it is $\frac{\bb_{m^*}(N_{m^*,\K(T)})}{1-\alpha}$. With expanding $\gamma_i^2 = \big (\frac{C_{m^*}}{(1-\alpha)C_i} + \frac{\beta}{(1-\alpha)C_i\sqrt{N_{m^*,\K(T)} \log(N_{m^*,\K(T)}/\delta)}}\big)^2$, and \ref{NoverN}, it is easy to show that,
\begin{align}
\label{term3}
    \frac{2N_{i,\K(T)}}{N_{m^*,\K(T)}} c_{h^*} \leq 2\gamma_i^2 \log(N_{m^*,\K(T)} /\delta)c_{h^*} \in O(\log(T/\delta))
\end{align}
And finally we know,
\begin{align}
\label{term4}
    \frac{\beta}{1-\alpha} \in O(\log(1/\delta)),
\end{align}
even if we put $\delta = T^\eta$ for some $\eta$ it will be in $O(\log(T))$. By gathering all parts \ref{term1},\ref{term3}, and \ref{term4} together, we have,
\begin{align}
    \Reg_i(N_{i,\K(T)}) \leq \frac{C_{m^*}^3 \sqrt{N_{m^*,\K(T)}} \log^2(N_{m^*,\K(T)} /\delta)}{(1-\alpha)^2 C_i^2 } + \frac{1}{1-\alpha}\bb_{m^*}(N_{m^*,\K(T)})+ O(\log^\frac{3}{2}(T/\delta)),
\end{align}
which concludes the proof.
\end{proof}

\subsection{Proof of Theorem \ref{maintheorem} (Main Theorem) }
\label{maintheorem proof}
\begin{theorem}
   By running the algorithm $\bear$ with base algorithms of $\pmdt$, over $M$ model classes $\cC_0$ to $\cC_{M-1}$, with the unknown optimal model class $\cC_{m^*}$, and known upper bound of $c_{h^*} \geq \spn(h^*)$, for all $0< \delta < 1$, with probability of at least $1- 26MT \delta$, the regret \ref{Reg} is upper bounded by,
\[
\Reg(\bear, T) \leq \big (\frac{m^* C_{m^*}^2 \log^{\frac{3}{2}}(T /\delta)}{(1-\alpha)^2 C_0^2 } + \frac{M}{1-\alpha} \big )\bb_{m^*}(T,\delta)+ O(\log^\frac{3}{2}(T/\delta)),
\]
where $\frac {1}{2} \leq \alpha =\frac{\log(c_{h^*}^5 \lor 9)+1}{2\log(c_{h^*}^5 \lor 9)} \leq \frac{3}{4}$.
\end{theorem}
\begin{proof}
The proof is similar to other works using the regret balancing idea \cite{abbasi2020balancebanditrl} \cite{pacchiano2020regretbalanceelim}. First of all

After proving the lemmas, we are now ready to prove the main theorem \ref{maintheorem}. We start with decompose the regret into the regrets of each model class,
\begin{align}
    \Reg(\bear, T)& := T g^*_{m^*} - \sumt R_t \\
    & = \sum_{i = 0}^{M-1} \big [ N_{i,\K(T)} \ g^\star - \sum_{t \in \cN_{i,\K(T)}} R_t \big ] \\
    & = \sum_{i = 0}^{M-1} \Reg_i(\balg_i, N_{i,\K(T)}).
\end{align}
Suppose we are under event $\cG$ which happens with the probability of at least $1 - 26(M-m^*)\K(T)\delta$. For all $i \in \cI_{K(T)+1}$ such that $i < m^*$ and , we can use \ref{finalReg_ibound} and write,
\begin{align*}
    \sum_{i = 0}^{m^*-1} \Reg_i(\balg_i, N_{i,\K(T)}) & \leq \sum_{i = 0}^{m^*-1} \Big ( \frac{C_{m^*}^3 \sqrt{N_{m^*,\K(T)}} \log^2(N_{m^*,\K(T)} /\delta)}{(1-\alpha)^2 C_i^2 } + \frac{1}{1-\alpha}\bb_{m^*}(N_{m^*,\K(T)})+ O(\log^\frac{3}{2}(T/\delta)) \Big) \\
    & \leq \frac{m^* C_{m^*}^3 \sqrt{N_{m^*,\K(T)}} \log^2(N_{m^*,\K(T)} /\delta)}{(1-\alpha)^2 C_0^2 } + \frac{m^*}{1-\alpha}\bb_{m^*}(N_{m^*,\K(T)})+ O(\log^\frac{3}{2}(T/\delta)).
\end{align*}
Note that $C_0 \leq C_1 \leq \dots C_{M-1}$ as the models get richer by increasing $i$, i.e. their regret guarantees increase. For the well-specified model classes $i \geq m^*$ we use their regret guarantees, and write,
\begin{align}
    \sum_{i = m^*}^{M-1} \Reg_i(\balg_i, N_{i,\K(T)}) & \leq \sum_{i = m^*}^{M-1} \bb_i( N_{i,\K(T)},\delta) \\
    & \leq  \bb_{m^*}(N_{m^*,\K(T)}) + \sum_{i = m^*+1}^{M-1} \big [ \frac{1}{1-\alpha}\bb_{m^*}(N_{m^*,\K(T)})+ \frac{\beta}{1-\alpha}\big] \\
    & \leq \frac{M-m^*}{1-\alpha}\bb_{m^*}(N_{m^*,\K(T)})+ \frac{(M-m^*-1)\beta}{1-\alpha},
\end{align}
where the first second inequlity is implied from lemma \ref{Regretisbalanced} (specifically \ref{temp6}), and the last inequality is from $\frac{1}{(1-\alpha)} \geq 1$ as  $1/2 \leq \alpha \leq 3/4$. Note that we are exploiting on the fact that visiting other model classes between the epochs of $\balg_i$ on $\cC_i$ does not affect on the regret bound guarantee \ref{baseguarantee}. This is clear from the proof and procedure of most of the algorithms, including $\pmdt$ that we use. Also, note that $\frac{(M-m^*-1)\beta}{1-\alpha} \in O(\log(1/\delta))$ which by choice of $\delta = T^\eta$ for some $\eta$ it will be in $O(\log(T))$. So we can put the the parts together and have,
\[
\Reg(\bear, T) \leq \frac{m^* C_{m^*}^3 \sqrt{N_{m^*,\K(T)}} \log^2(N_{m^*,\K(T)} /\delta)}{(1-\alpha)^2 C_0^2 } + \frac{M}{1-\alpha}\bb_{m^*}(N_{m^*,\K(T)},\delta)+ O(\log^\frac{3}{2}(T/\delta))
\]
which by applying $N_{m^*,\K(T)} \leq T$ and changing $\K(T) \leq T$ concludes the proof.
\end{proof}

\subsection{Proof of Lower Bound} \label{proof of lowerbound}
We modify the Le Cam method described in the lower bound section of \cite{lattimore2020bandit} for reinforcement learning and especially for our setting instead of bandit problem. We need the following two lemmas before designing the opponent's policies.
\begin{lemma}
\label{Divergence Decomp}
    (Divergence Decomposition) Fix an algorithm $Alg$. Let $\psi$ and $\psi'$ be two opponent's policies and $\bP_\psi$ (resp. $\bP_{\psi'}$) be the probability measures on the trajectories of learner and opponent actions, induced by the interaction of $Alg$ with $\psi$ (resp. $\psi'$) for $T$ rounds. Then,
    \begin{align}
        D_{KL}(\bP_\psi \Vert \bP_{\psi'}) = \sum_{s\in\cS} \lambda_\psi(s)D_{KL}(\psi(s) \Vert \psi'(s))
    \end{align}
    where $\psi(s) \in \Delta_\cB$ is the distribution from which $\psi$ takes action in state $s$, and$\lambda_\psi(s) = \sumt \Prob_\psi [S_t = s ]$ is the occupancy measure induced over the states by $\psi$ and $Alg$ after $T$ rounds.
\end{lemma}
\begin{proof}
Assume that for $s \in \cS$ we have $D_{KL}(\psi(s) \Vert \psi'(s)) < \infty$. The algorithm $Alg$ implements a memory-full policy $\pi_t$ at round $t$.  We can write
\begin{align}
    \bP_\psi(a_1,b_1,\dots,a_T,b_T) & = \Pi_{t =1}^T \pi_t(a_t | a_1,b_1,\dots,a_{t-1},b_{t-1})\Prob_\psi(b_t| a_1,b_1,\dots,a_{t-1},b_{t-1}) \\
    & = \Pi_{t =1}^T \pi_t(a_t | a_1,b_1,\dots,a_{t-1},b_{t-1})\psi(b_t| s_t)
\end{align}
By chain rull for Radon-Nikodym derivatives we have,
\[
\log \frac{d \bP_\psi}{d \bP_{\psi'}}(a_1,b_1,\dots,a_T,b_T) = \sumt \log \frac{\psi(b_t|s_t)}{\psi'(b_t|s_t)}
\]
And by taking expectations,
\[
\E_\psi [\log \frac{d \bP_\psi}{d \bP_{\psi'}}(A_1,B_1,\dots,A_T,B_T) ]= \sumt \E_\psi [\log \frac{\psi(B_t|S_t)}{\psi'(B_t|S_t)}],
\]
because $b_t$ is only be determined by $\psi$, the distribution of $B_t$ under $\Prob_\psi (.|S_t)$
is the same as $\psi(.|S_t)$. This together with the tower rule gives us,
\[
\E_\psi [\log \frac{\psi(B_t|S_t)}{\psi'(B_t|S_t)}] = \E_\psi \big [ \E_\psi [\log \frac{\psi(B_t|s_t)}{\psi'(B_t|s_t)} | S_t=s_t]\big ] = \E_\psi [D_{KL}(\psi(S_t) \Vert \psi'(S_t))].
\]
Now we can write,
\begin{align}
D_{KL}(\bP_\psi \Vert \bP_{\psi'}) &= \E_\psi [\log \frac{d \bP_\psi}{d \bP_{\psi'}}(A_1,B_1,\dots,A_T,B_T) ] \\
& = \sumt \E_\psi [D_{KL}(\psi(S_t) \Vert \psi'(S_t))] \\
& = \sumt \E_\psi [\sum_{s\in\cS} \1\{S_t = s\}D_{KL}(\psi(s) \Vert \psi'(s))] \\
& = \sum_{s\in\cS} \lambda_\psi(s)D_{KL}(\psi(s) \Vert \psi'(s)).
\end{align}
When $s \in \cS$ is occupied by $\psi$ with a positive measure and $\log \frac{\psi(b|s)}{\psi'(b|s)}$ is infinite for some $b \in \cB$, then there a trajectory with a positive probability that makes $\log \frac{d \bP_\psi}{d \bP_{\psi'}}(a_1,b_1,\dots,a_T,b_T)$ also infinite, meaning the lemma holds for the case of infinity as well.


\end{proof}
\begin{lemma}
    For two distribution vectors $P$ and $Q$ in $\Delta_{\bB}$, such that 
\begin{align*}
    \begin{cases}
        P_1 = P_2 = 1/2 \\
        P_i = 0 \quad \forall i \in [3:\bB]
    \end{cases}
\end{align*}
and
\begin{align*}
    \begin{cases}
        Q_1 = 1/2 - 2 \epsilon \\
        Q_2 = 1/2+ 2 \epsilon \\
        Q_i = 0 \quad \forall i \in [3:\bB]
    \end{cases}
\end{align*}
$D_{KL}(P \Vert Q) = \frac{1}{2}(\log(\frac{1}{1-4\epsilon}) + \log(\frac{1}{1+4\epsilon})) = 8 \epsilon^2 + c \epsilon^4$ for $\epsilon \in (0,1/4)$ and some constant $c$.
\end{lemma}
\begin{proof}
    From the definition of Kullback-Leibler divergence we have, 
    \begin{align*}
        D_{KL}(P\Vert Q) = \sum_i P_i \log(\frac{P_i}{Q_i}) = \frac{1}{2}(\log(\frac{1}{1-4\epsilon}) + \log(\frac{1}{1+4\epsilon})),
    \end{align*}
    since the only non-zero terms are the first and the second terms. Now write the Taylor expansion for the two terms of $\log(\frac{1}{1-4\epsilon})$ and $\log(\frac{1}{1+4\epsilon})$ at point $\epsilon = 0$,
    \begin{align}
        \log(\frac{1}{1-4\epsilon}) = \epsilon [\frac{4}{1-4\epsilon}]_{\epsilon = 0} + \frac{\epsilon^2}{2!} [\frac{16}{(1-4\epsilon)^2}]_{\epsilon = 0} + \frac{\epsilon^3}{3!} [\frac{128}{(1-4\epsilon)^3}]_{\epsilon = 0} + c \epsilon^4,
    \end{align}
and 
   \begin{align}
        \log(\frac{1}{1+4\epsilon}) = \epsilon [\frac{-4}{1+4\epsilon}]_{\epsilon = 0} + \frac{\epsilon^2}{2!} [\frac{16}{(1+4\epsilon)^2}]_{\epsilon = 0} + \frac{\epsilon^3}{3!} [\frac{-128}{(1+4\epsilon)^3}]_{\epsilon = 0} + c \epsilon^4.
    \end{align}
    Therefore by summing the two parts, we conclude the statement in the lemma.
\end{proof}
Now we need to define special sequences that will be of use in designing communicating opponent's policies. These special series are called de Bruijn sequences \cite{de1975acknowledgement,crochemore2021problems125}.
\begin{definition}
    The sequence $b_{\bB^m} \dots b_2 b_1$ is called an $m$-th order \emph{de Bruijn} sequence on alphabet $\cB$ with the cardinality of $\bB = |\cB|$, if it contains every ordered tuple of actions in length $m$. In other words,
    \[
    \forall \ (b^m,b^{m-1},\dots,b^1) \in \cB^m \quad \exists t \in [\bB^m]: \quad (b^m,b^{m-1},\dots,b^1) = (b_t, b_{t-1}, \dots, b_{t-m+1})
    \]
    where the indexes are cyclic.
\end{definition}
This is a sequence that visits all the permutations of length $m$ from alphabet $\cB$ without any repetition. As an example, for $\cB = \{0,1 \}$ and $m = 2$, the sequence $0110$ is a $2$nd-order de Bruijn sequence. Note that the sequence is cyclic. These sequences exist for all order $m$ and alphabet size $\bB$ \cite{de1975acknowledgement, crochemore2021problems125}.


\noindent{\textbf{Designing the Game and $\psi$}.}
\begin{theorem}
    Suppose the number of opponent's actions $|\cB |\geq 3$ and number of learner's actions $|\cA| \geq 2$. For any fixed memory $m \geq 2$ known for the learner, and any algorithm $Alg$, there exists a stage game with utility $U : \cA \times \cB \rightarrow [0,1]$, and a general opponent's policies $\psi_{Gen}$ such that
    \[
    \sR_{\psi_{Gen}}(Alg,T) \in \Omega(\frac{1}{m} \sqrt{\bA^{m-1}\bB^{m-1} T}).
    \]
\end{theorem}
\begin{proof}
Consider the set of learner's actions $\cA = \{a_1 , \dots,a_{\bA-1} \} \cup \{ a^*\}$ and the set of opponent's actions $\cB = \{b_1 , \dots,b_{\bB-2} \} \cup \{ b^* , b_r\}$. Also consider a utility function for the learner such that $\bU(a^*,b^*) = 1$ for an special learner's action $a^*$ and an special opponent's action $b^*$, and $\bU(a,b) = 0$ for all other entries, i.e. $a \neq a^*$ or $b\neq b^*$. Therefore the utility matrix has only one entry with the reward of $1$ and $0$ every where else. Fix a de Bruijn sequence $\fb_{\bB-2^{(m-1)}} \dots \fb_2 \fb_1$ over the alphabet $\{b_1 , \dots,b_{\bB-2} \}$ in the order of $m-1$. We define a special $m-1$-th order state $s^* = (a^*_{m},b^*_{m},\dots,a^*_2,b^*_2) \in (\cA \backslash \{ a^*\} \times \cB \backslash \{ b^*, b_r\})^{m-1}$, and based on that we construct a rewarding state space $\cS_R$, as the set of all states in the form of
\begin{align*}
    \cS_R = \{ s \in \cS_m: s = (s^*,a_1,b_1)\quad  \forall (a_1,b_1) \in \cA \times \cB \}.
\end{align*}
In other words, $s^*$ is a special state in the order of $m-1$ and if a state $s$ has the prefix of $s^*$ in its previous actions, no matter what the last actions ($m$-th previous interaction) are, it will be added to rewarding space $\cS_R$. Define $\dB : \cS_{m-1} \rightarrow \cB$ as a function that gets an $m-1$-th order state $s = (a_{m},b_{m},\dots,a_2,b_2)$ as an input, then search for the sequence of $(b_m,\dots,b_2)$ in the de Bruijn sequence $\fb_{\bB-2^{(m-1)}} \dots \fb_2 \fb_1$, and outputs the next character in the de Bruijn sequence. i.e. if $(\fb_{t+m} \dots \fb_{t+2} ) = (b_m,\dots,b_2)$, then $\dB(s) = \fb_{t+m+1}$. This is an increment operator on the de Bruijn sequence. 
We are now ready to define the following opponent's policy $\psi$,
\begin{align}
    \begin{cases}
        \psi(b^* | (s^*, a,b)) = 1/2+ \epsilon \quad \quad \forall a,b \in \cA \backslash \{a^* \} \times\cB \\
        \psi(\dB(s^*) | (s^*, a,b)) = 1/2- \epsilon \quad \quad \forall a,b \in \cA \backslash \{a^* \}\times\cB\\
        \psi(b^* | (s, a,b)) = 1/2 \quad \quad  \forall a,b \in \cA \backslash \{a^* \} \times\cB \ , \ \forall  s \neq s^*\ \text{and} \ s \in (\cA \backslash \{ a^*\} \times \cB \backslash \{ b^* ,b_r\})^{m-1}\\
        \psi(\dB(s^*) | (s, a,b)) = 1/2 \quad \quad  \forall a,b \in \cA \backslash \{a^* \} \times\cB \ , \ \forall  s \neq s^*\ \text{and} \ s \in (\cA \backslash \{ a^*\} \times \cB \backslash \{ b^*,b_r\})^{m-1}\\
        \psi(b | (s, a,b)) = 1 \quad \quad  \forall a,b \in \cA \backslash \{a^* \} \times\cB \ , \ \forall s \notin (\cA \backslash \{ a^*\} \times \cB \backslash \{ b^*,b_r\})^{m-1}\\
        \psi(b^* | (s^*, a^*,b)) = 1/2+ \epsilon \quad \quad \forall b \in \cB \\
        \psi(b_r | (s^*, a^*,b)) = 1/2- \epsilon \quad \quad \forall b \in \cB\\
        \psi(b^* | (s, a^*,b)) = 1/2 \quad \quad  \forall b \in \cB \ , \ \forall  s \neq s^*\ \text{and} \ s \in (\cA \backslash \{ a^*\} \times \cB \backslash \{ b^*,b_r\})^{m-1}\\
        \psi(b_r| (s, a^*,b)) = 1/2 \quad \quad  \forall b \in  \cB \ , \ \forall  s \neq s^*\ \text{and} \ s \in (\cA \backslash \{ a^*\} \times \cB \backslash \{ b^*,b_r\})^{m-1}\\
    \end{cases}
\end{align}

This choice of $\psi $ considers the previous $m-1$ interactions, and based on the opponent's actions it goes to the next action suggested by the de Bruijn policy with probability $\frac{1}{2}$. Or it takes a special action $b^*$ with probability $\frac{1}{2}$. In this case, if the learner would have taken $a^*$ then it gained a reward of $1$. In case of taking $a^*$ or $b^*$For the following $m-1$ iterations, the opponent takes the last action of the state deterministically, so that it makes a circular return to that state. In an special state $s^*$ the probability of taking $b^*$ is $\frac{1}{2} + \epsilon$, which implies the optimal strategy to be reaching to $s^*$ and then taking action $a^*$ and then stay in this state $s^*$ by taking action in $s^*$. When the $a^*$ reaches the $m$-th previous action, then  $\psi$ probabilistically choose $b^*$ or $b^r$,i.e. the action that helps in repetition of the state. The optimal gain for this policy is $\frac{1}{2m} + \frac{\epsilon}{m}$, as in each $m$ iteraitons it gets $\frac{1}{2} +\epsilon$ reward on expectation. So the optimal policy iterates in the state space of,
$\cG_\psi = \{ s \in \cS_m : s = (a^*_m,b^*_m, \dots a^*_m,b^*_m , a^*, b_r) , s = (a^*_m,b^*_m, \dots a^*_m,b^*_m , a^*, b^*) \text{and their rotations}\}$
So the size of $|\cG_\psi| \leq 2m$. Now suppose an algorithm $Alg$ playing with $\psi$, induces an occupancy measure $\lambda_\psi(s) = \sumt \Prob_\psi [S_t = s ]$. We denote an state $s' \in  (\cA \backslash \{ a^*\} \times \cB \backslash \{ b^*, b_r\})^{m-1} \backslash \cG_\psi$ that has the least occupancy measure. That will be the Achilles' heel of the algorithm. According to this choice, $\lambda_\psi(s') \leq \frac{T}{C_L}$ where $C_L = (\bA -1)^{m-1}(\bB -2)^{m-1} - 2m $. Now suppose another opponent's policy $\psi'$ which is exactly identical to $\psi$, except it has the probability of taking $b^*$ equal to $\frac{1}{2} + 2 \epsilon$ in state $s'$. So the optimal policy in $\psi'$ is iterating on $s'$ instead of $s^*$, and it achieves optimal gain of $\frac{1}{2m} + \frac{2 \epsilon}{m}$. Now we can write,
\begin{align*}
    \Reg_\psi(Alg,T) \geq \bP_\psi[\mathsf{T}(\cG_\psi) \leq \frac{T}{2}]\frac{T \epsilon}{2m},
\end{align*}
where $\mathsf{T}(\cG)$ is the number of iterations that a state $s\in\cG$ is occupied. The inequality holds because when at least $T/2$ of iterations are not in the states of $\cG_\psi$, for each iteration we are suffering $\epsilon/m$ regret on average. The same can be said for $\psi'$,
\begin{align*}
    \Reg_{\psi'}(Alg,T) \geq \bP_{\psi'}[\mathsf{T}(\cG_\psi) \geq \frac{T}{2}]\frac{T \epsilon}{2m},
\end{align*}
because each iteraion consumed in $\cG_\psi$ is accompanied by a regret of $\epsilon/m$ on average when the opponent is playing $\psi'$. So from Bretagnolle–Huber in equality we have,
\begin{align}
     \Reg_\psi(Alg,T) +  \Reg_{\psi'}(Alg,T) & \geq  (\bP_\psi[\mathsf{T}(\cG_\psi) \leq \frac{T}{2}] + \bP_{\psi'}[\mathsf{T}(\cG_\psi) \geq \frac{T}{2}]) \ \frac{T \epsilon}{2m}\\
     & \geq \frac{1}{2} \exp(- D_{KL}(\bP_\psi \Vert \bP_{\psi'}))\frac{T \epsilon}{2m}.
\end{align}
Now from the two previous lemmas, we have,
\begin{align}
\Reg_\psi(Alg,T) +  \Reg_{\psi'}(Alg,T) & \geq \frac{1}{2} \exp(- D_{KL}(\bP_\psi \Vert \bP_{\psi'}))\frac{T \epsilon}{2m} \\
& \geq \frac{1}{2} \exp(- \E_\psi (\mathsf{T}(s'))D_{KL}(\psi(s')\Vert \psi'(s')))\frac{T \epsilon}{2m} \\
& \geq \frac{1}{2} \exp(- \frac{T}{C_L} (8\epsilon^2 + c \epsilon^4))\frac{T \epsilon}{2m},
\end{align}
which means by choosing $\epsilon = \sqrt{C_L /T} \leq 1/4$ for big enough $T$, we have $\exp(- \frac{T}{C_L} (8\epsilon^2 + c \epsilon^4))$ less than a constant and,
\begin{align*}
    \Reg_\psi(Alg,T) +  \Reg_{\psi'}(Alg,T) \geq \frac{C}{m} \sqrt{T ((\bA -1)^{m-1}(\bB -2)^{m-1} - 2m) }
\end{align*}
so for at least one of the regrets we have,
\begin{align*}
\Reg(Alg,T) \in \Omega (\frac{1}{m} \sqrt{\bA^{m-1}\bB^{m-1} T }).
\end{align*}
\end{proof}

\end{document}